\documentclass{jair}

\setcopyright{cc}
\copyrightyear{2026}
\acmDOI{10.1613/jair.1.21969}

\JAIRAE{Mikko Koivisto}
\JAIRTrack{} % Insert JAIR Track Name only if part of a special track
\acmVolume{87}
\acmArticle{15}
\acmMonth{10}
\acmYear{2026}

\RequirePackage[
  datamodel=acmdatamodel,
  style=acmauthoryear,
  backend=biber,
  giveninits=true,
  uniquename=init
  ]{biblatex}

\usepackage{amsthm}

\usepackage{todonotes}
\theoremstyle{definition}
\newtheorem{definition}{Definition}
\theoremstyle{definition}

\theoremstyle{plain}
\newtheorem{theorem}{Theorem}
\theoremstyle{remark}
\newtheorem{remark}{Remark}

\theoremstyle{plain}
\theoremstyle{definition}
\newtheorem{fact}{Fact}
\theoremstyle{plain}
\newtheorem{proposition}{Proposition}
\theoremstyle{plain}
\newtheorem{lemma}{Lemma}
\theoremstyle{definition}
\newtheorem{example}{Example}

\usepackage{caption}
\usepackage{lipsum,framed}
\usepackage{amsmath}
\usepackage{pifont}% http://ctan.org/pkg/pifont
\usepackage{comment}
\usepackage{tikz}
\usetikzlibrary{positioning,arrows, calc,shapes,babel,fit,positioning} 
\usetikzlibrary{decorations.pathmorphing}
\tikzset{
modal/.style={>=stealth',shorten >=1pt,shorten <=1pt,auto,node distance=1.5cm,
semithick},
world/.style={circle,draw,minimum size=0.5cm},
arg/.style={circle,draw,minimum size=0.5cm},
sarg/.style={draw},
carg/.style={draw,minimum size=0.5cm},
unarg/.style={circle,draw,minimum size=0.5cm,dashed},
fixarg/.style={circle,draw,minimum size=0.5cm},
point/.style={circle,draw,inner sep=0.5mm,fill=black},
reflexive above/.style={->,loop,looseness=7,in=120,out=60},
reflexive below/.style={->,loop,looseness=7,in=240,out=300},
reflexive left/.style={->,loop,looseness=7,in=150,out=210},
reflexive right/.style={->,loop,looseness=5,in=30,out=330},
coil/.style={decorate, decoration={coil,amplitude=4pt,segment length=5pt}},
snake/.style={decorate, decoration={snake}},
zigzag/.style={decorate, decoration={zigzag}}
}

\newcommand{\ang}[1]{\langle #1 \rangle}
\newcommand{\tuple}[1]{\langle #1 \rangle}

\newcommand{\uni}{\mathcal{U}}
\newcommand{\args}{\mathcal{A}}
\newcommand{\rel}{\mathcal{D}}
\newcommand{\af}{\ang{\args,\rel}}		% arg.framework AF

\newcommand{\imparg}{\ang{\args^{\fix},\args^{?},\defrel,\Delta}}

\newcommand{\lexpressive}{\preccurlyeq}

\newcommand{\compfun}{\mathsf{com}}
\newcommand{\niaf}{\mathsf{IAF}}
\newcommand{\nargiaf}{\mathsf{arg}\text{-}\mathsf{IAF}}

\newcommand{\nimpargiaf}{\mathsf{imp}\text{-}\mathsf{arg}\text{-}\mathsf{IAF}}

\newcommand{\comp}{\ang{\args^{\ast},\defrel^{\ast}}}
\newcommand{\naf}{\mathsf{AF}}

\newcommand{\impor}{\mathsf{IMPLY}^{\lor}}
\newcommand{\opimp}{\mathsf{IMPLY}}
\newcommand{\opor}{\mathsf{OR}}
\newcommand{\opnand}{\mathsf{NAND}}

\newcommand{\ndepargiaf}{\mathsf{dep}\text{-}\mathsf{arg}\text{-}\mathsf{IAF}}

\newcommand{\fix}{F}

\newcommand{\argsf}{\args^{\fix}}

\newcommand{\lanset}{\mathcal{L}}
\newcommand{\rulset}{\mathsf{R}}
\newcommand{\namefun}{\mathfrak{n}}

\newcommand{\defrel}{\mathcal{D}}
\newcommand{\negfun}{\overline{\cdot}}

\newcommand{\at}{\mathsf{AT}}

\newcommand{\unravelat}{\ang{\lanset, \negfun, \rulset, \namefun,\kb}}

\newcommand{\kb}{\mathcal{K}}

\newcommand{\sub}{\mathsf{Sub}}
\newcommand{\prem}{\mathsf{Prem}}
\newcommand{\conc}{\mathsf{Conc}}
\newcommand{\ftoprule}{\mathsf{TopRule}}
\newcommand{\sto}{\!\twoheadrightarrow\!}
\newcommand{\rulprefix}{\mathsf{rul}\text{-}}

\newcommand{\premprefix}{\mathsf{prem}\text{-}}
\newcommand{\argprefix}{\mathsf{arg}\text{-}}

\newcommand{\nsaf}{\mathsf{SAF}}
\newcommand{\daf}{\mathsf{AF}}

\newcommand{\nisaf}{\mathsf{ISAF}}

\newcommand{\unravelsaf}{\ang{\lanset, \negfun, \rulset, \namefun, \kb, \preceq}}
\newcommand{\indexedsaf}[1]{\ang{\lanset_{#1}, \negfun_{#1}, \rulset_{#1}, \namefun_{#1}, \kb_{#1}, \preceq_{#1}}}
\newcommand{\unravelsafprime}{\ang{\lanset, \negfun^\ast, \rulset^{\prime}, \namefun, \kb, \preceq^{\ast}}}

\newcommand{\asset}{\kb}
\newcommand{\contfun}{\overline{\cdot}}
\newcommand{\abaf}{\mathsf{ABAF}}
\newcommand{\ddaf}{\mathsf{DAF}}
\newcommand{\iabaf}{\mathsf{IABAF}}
\newcommand{\unravelabaf}{\tuple{\lanset,\rulset,\asset,\contfun}}

\newcommand{\assprefix}{\mathsf{asu}\text{-}}

\newcommand{\restrict}{\upharpoonright}

\begin{document}

%\title{JAIR Example Template}
\title[Structured Argumentation Frameworks with Uncertain Rules and Premises]{Expressivity Results for Structured Argumentation Frameworks with Uncertain Rules and Premises}

%% used to denote shared contribution to the research and/or corresponding author.
\author{Carlo Proietti}
%\authornote{Corresponding Author.}
\orcid{0000-0001-9044-9642}
\email{carlo.proietti@ilc.cnr.it}
\affiliation{%
  \institution{Consiglio Nazionale delle Ricerche (CNR)}
  \city{Genova}
  %\state{British Columbia}
  \country{Italy}
}

\author{Antonio Yuste-Ginel}
\authornote{Corresponding Author.}
\orcid{0000-0002-4380-3095}
\email{ayusteginel@uma.es}
\affiliation{%
  \institution{University of M\'alaga (UMA)}
  \city{M\'alaga}
  %\state{British Columbia}
  \country{Spain}}

%% The short list of authors must be made of the list of all authors' lastnames.
\renewcommand{\shortauthors}{Proietti \& Yuste-Ginel}
%% If this is too long and overlaps other information printed in the page headers, use
%\renewcommand{\shortauthors}{Xu et al.}

\begin{abstract}
Modelling qualitative uncertainty in formal argumentation is essential both for practical applications and theoretical understanding. Yet, most existing work focuses on \textit{abstract} models of argumentation under uncertainty. Following a recent trend in the literature, we tackle the open question of studying plausible instantiations of these abstract models. To do so, we root the uncertainty of arguments in their components in structured models: rules and premises. Our main conceptual contribution is to provide an intuitive notion of comparative expressivity suitable for both abstract and structured formalisms. { Roughly, one formalism for arguing with qualitative uncertainty is at least as expressive as another if the former can generate the same set of abstract argumentation frameworks (completions) as the latter.}
On the technical side, we prove negative and positive results comparing the expressivity of abstract and structured models of argumentation with uncertainty. These results concern incomplete abstract argumentation frameworks and their extension with dependencies, on the abstract side, and ASPIC$^+$, assumption-based argumentation (ABA) and deductive argumentation on the structured side.
\end{abstract}

%% JAIR Note: 
%% Do not include ACM CCS Concepts or Keywords

%% To be updated by authors.
\received{17 February 2026}
\received[accepted]{17 September 2026}

\maketitle

\sloppy

%=====================================
\section{Introduction}
\label{sec:intro}
%=============

Uncertainty plays an important role in the construction and use of arguments. Considering different sets of background assumptions or different inference rules as applicable determines which conclusions are allowed. In debates, awareness of the inferential skills of an opponent (or an audience) and of what facts they take for granted
funnels the use of specific arguments.
From a broader Artificial Intelligence perspective, the analysis of uncertainty in argumentation constitutes a central problem in Knowledge Representation and Reasoning. Recent years have witnessed increasing interest in interactive AI systems—such as intelligent debaters, argumentative dialogue systems, and AI assistants—designed to engage users through explanation, persuasion, or deliberation \cite{gurevych2016debating,bar-haim2021advances, slonim2021autonomous}. Such systems must operate under conditions of incomplete information, heterogeneous reasoning standards, and partially shared background knowledge \cite{habernal2017argumentation,ruggeri2023dataset}. Consequently, they require principled mechanisms for representing and reasoning about uncertainty in argumentative contexts, where the strength and acceptability of arguments depend not only on logical structure but also on agents’ assumptions and inferential competences.\par 

\par

In computational argumentation, there are two modelling paradigms, with different capacities to represent uncertainty. The first one is abstract argumentation \cite{dung1995acceptability}, which is formally based on argumentation frameworks, i.e. directed graphs representing arguments (the nodes) and their
defeats (the arrows), and where uncertainty can be about these two entities. The other is structured argumentation, instantiated by various formalisms such as ASPIC$^{+}$ \cite{modgil2013aspic,modgil2014tutorial}, assumption-based argumentation (ABA) \cite{bondarenko1997aba, toni2014tutorial}, deductive argumentation \cite{besnard2001logic,besnard2014constructing}, and defeasible logic programming (DeLP) \cite{garcia2014defeasible,garcia2004defeasible}. Here, the fundamental entities are argument components, such as premises and inference rules, and uncertainty may concern any of them.

\par
The general research question that we want to tackle is how uncertainty compares in abstract and structured argumentation. Our main aim here is to define a systematic approach to enable such a comparison, which in principle allows us to map the relative expressivity of different formalisms for handling uncertainty. Establishing general expressivity results is a way to address many issues about the use of abstract argumentation in structured contexts. In fact, the most relevant concern is that the information contained in structured frameworks is often lost when moving to their abstract counterpart, insofar as many argumentation frameworks are not plausible instantiations of structured ones \cite{prakken2018abstraction}. This holds for many specifications of abstract argumentation, such as preference-based and value-based, gradual, bipolar, and probabilistic argumentation. This motivated a new branch in the formal argumentation literature that consists of revisiting abstract frameworks to find the structured counterpart; see, among others, \cite{cohen2018characterization,rapberger2023,prakken2024dynamics}. 

\par
Our general method involves three steps. First, we extend and systematise previous characterisations of structured argumentation formalisms with uncertain components. Our characterisations are built upon the ASPIC$^{+}$ framework but are straightforwardly transferable to other formalisms such as ABA or deductive argumentation. 
Second, we set a standard framework for uncertainty in abstract argumentation.
The one we employ here is \emph{argument-incomplete argumentation frameworks with dependencies} (dep-arg-IAFs) \cite{fazzingaijcai21}. In this context, uncertainty is encoded by the set of arguments that are left undecided, and expressivity is measured by the number of possible \textit{completions}, i.e. the ways uncertainty can be settled. More precisely, and within this paper, a completion is just a directed graph (i.e., an abstract argumentation framework). The third step consists of ``lifting'' a specific structured formalism with uncertainty to the abstract level. This means endowing structured formalisms with an abstract counterpart and then comparing them with different subclasses of dep-arg-IAFs in terms of their \textit{relative expressivity}. 
\par

\par 
The notion of ``expressivity'' adopted in this paper 
is motivated by the question, stated above, of whether abstract argumentation formalisms provide an adequate tool for ``expressing'' structured ones, which is a potential challenge to the widely adopted view that the most relevant aspects of an argumentative situation can be represented by a directed graph (as originally proposed by Dung). Within this perspective, asking which forms of qualitative uncertainty a formalism can capture amounts to asking which sets of abstract argumentation frameworks it can represent. The proposed notion of expressivity captures this representational power in this precise sense, by generalising previous proposals in the literature \cite{mailly2022yes,jlc,mailly2024constrainedjournal} in order to deal with structured formalisms. Roughly speaking, we say that a formalism $F$ is at least as expressive as another $G$
if for each instance of $G$ we can find one of $F$ with an equivalent set of completions. 
Our notion of equivalence (Section \ref{sec:def-expressivity}) is basically a lifting of the notion of isomorphism from graphs to sets of graphs. Hence, any assertion about the relative expressivity of $F$ and $G$ assumes that both classes of formalisms come equipped with an abstraction mechanism (a way to extract directed graphs from their components when there is no uncertainty) and a completion function (a way of extracting a set of directed graphs from their components when they have uncertain components). However, this dependence does not preclude the applicability of our notion; we can study expressivity maps for formalisms using different abstraction and completion mechanisms. Importantly, our notion of equivalence (and therefore of expressivity) is independent of the chosen argumentation semantics (hence, it can be deemed ``structural'' or ``syntactic'').
\par 

Our expressivity results are mainly relative to the ASPIC$^{+}$ formalism and are presented in Section \ref{sec:aspic-with-uncertainty} (summarised in Figure \ref{fig:arg-ISAFs-results}). Essentially, they show that ASPIC$^{+}$ with uncertainty is strictly more expressive than general argument-incomplete argumentation frameworks without dependencies, but strictly less than frameworks with dependencies. They also show that rule-incomplete ASPIC$^{+}$ systems are strictly more expressive than premise-incomplete ones. The part of these results providing lower and upper bounds of expressivity extends to ABA and (a specific instance of) deductive argumentation. However, some of their distinctive features prevent the complete transfer of the results from ASPIC$^{+}$.

\par 
We proceed as follows. Section \ref{sec:stateoftheart} illustrates the state of the art of formalisms for uncertainty in abstract and structured argumentation. In Section \ref{sec:background}, we provide the basics of abstract argumentation and its extension to incomplete argumentation frameworks with argument-incompleteness. Further, we introduce the fundamental notion of \textit{relative expressivity} among classes of incomplete AFs, which is key to comparing structured and abstract formalisms. In Section \ref{sec:aspic-with-uncertainty}, we introduce ASPIC$^{+}$ as our reference framework for structured argumentation and endow it with two types of uncertainty: premise-incompleteness and rule-incompleteness. We then proceed to compare the expressivity of different types of dep-arg-IAFs with incomplete ASPIC$^{+}$ frameworks. 
Section \ref{sec:other-frameworks} shows that our characterisation of uncertainty and the method of lifting structured into abstract formalisms can be straightforwardly extended to ABA and deductive argumentation. However, the bulk of our expressivity results on ASPIC$^{+}$ does not extend automatically to these formalisms due to some subtle differences in the definition of argument. Here we discuss these issues and formulate some conjectures, leaving their analysis for future investigation. 
Section \ref{sec:conclusion} concludes the paper.

%============================
\section{State of the Art and Contribution}\label{sec:stateoftheart}
%============================

Abstract argumentation has argumentation frameworks as its basic formal tool and extensions thereof. 
 Recent literature on abstract argumentation includes models based on different approaches to representing uncertainty, either in a probabilistic fashion \cite{hunter2021survey} or through qualitative methods \cite{mailly2022yes}, that is, without a numerical representation.
The qualitative frameworks share a common idea for representing uncertainty, namely that of working with a plurality of argumentation frameworks, called \textit{completions}, as possible worlds.

In this paper we use  \emph{incomplete abstract argumentation frameworks} (IAFs) as our abstract model of qualitative uncertainty \cite{cayrol2007partial,arg-IAFs-original,baumeister2018verification,baumeister2021acceptance}, and restrict our focus to the case where incompleteness concerns only the arguments (arg-IAFs) and not their defeats (def-IAFs), as they subsume almost all relevant forms of structured argumentation, except for two cases that we discuss in Section \ref{sec:conclusion}. 
As mentioned, to achieve an adequate level of granularity, we employ an extension of arg-IAFs, namely \emph{arg-IAFs with dependencies} \cite{fazzingaijcai21} (see \cite{fazzingakr21} for the defeat-incomplete version of this formalism). Intuitively, dependencies are specific constraints on which arguments should or should not be present in a completion (see Section \ref{sec:background}).  
It has been shown elsewhere \cite{fazzingaijcai21,fazzingakr21} that finite IAFs with dependencies can generate every subset of possible completions. Here we straightforwardly extend this result to the infinite case (Section \ref{sec:background}) as needed for the present context, where structured argumentation systems easily generate infinite sets of arguments. Moreover, we show that to express structured argumentation systems it is sufficient to employ a strict subclass of dep-IAFs, consisting of IAFs with only \emph{implicative} dependencies among arguments, thus establishing a lower threshold for expressivity, and therefore a lower computational complexity for the associated reasoning tasks in the finite case. In particular, lower complexity was shown for arg-IAFs with a particular kind of implicative dependency (the one we will use below) and for the possible verification task under admissible and stable semantics \cite{fazzingaijcai21}.

\par
Unlike abstract, structured argumentation is instantiated by a plurality of baseline formalisms such as ASPIC$^{+}$ \cite{modgil2013aspic,modgil2014tutorial}, assumption-based argumentation (ABA) \cite{bondarenko1997aba,toni2014tutorial}, deductive argumentation \cite{besnard2001logic,besnard2014constructing}, and defeasible logic programming (DeLP) \cite{garcia2004defeasible,garcia2014defeasible}.
All these approaches share two fundamental entities: premises, which may be divided into defeasible (assumptions) and non-defeasible ones (facts), and inferential rules, which are also divided into non-defeasible (strict) and defeasible ones. In such frameworks, uncertainty may concern elements of these categories, giving rise to either premise-incompleteness or rule-incompleteness. ASPIC$^{+}$ also involves preferences among arguments, where uncertainty may apply to the way arguments are ranked. Although preference-incompleteness can be dealt with in the same way as premise- and rule-incompleteness, we defer its treatment to a future occasion, briefly discussing it in Section \ref{sec:conclusion}, to keep our approach more compact. 

 Compared to abstract argumentation, attempts to model uncertainty in structured argumentation are more recent and, to date, less systematic \cite{ai32023,odekerken2025argumentative,ria2024}.
One contribution of this paper is to provide a uniform and systematic treatment of such formalisms. In this sense, the present work expands and consolidates previous work \cite{odekerken2025argumentative,ai32023,ria2024}, which addresses only some of the instances considered in Section~\ref{sec:aspic-with-uncertainty}, and exclusively in the context of ASPIC$^+$. 
 More precisely, Definitions~\ref{def:rul-ISAFs} and~\ref{def:rul-isaf-completions} (rule-incomplete frameworks and their completions) were first introduced by \textcite{ai32023}. That work also contains Proposition~\ref{thm:rul-ISAFs-n-less-expressive-arg-IAFs} and a defective version of Theorem~\ref{thm:rul-iafs-representation}. A simpler version of Definitions~\ref{def:prem-isaf} and~\ref{def:prem-isaf-completion} (premise-incomplete models and their completions) was introduced by \textcite{odekerken2025argumentative}; that version only considers strict premises and defeasible rules over a propositional language of literals. Importantly, apart from the just-mentioned result, the present paper does not reuse substantive results from previous work, but rather builds upon and extends the underlying definitions.

\par 
All remaining definitions, constructions, and results in Section~\ref{sec:aspic-with-uncertainty}, summarised in Figure~\ref{fig:arg-ISAFs-results}, 
are new to this paper.
These definitions can be exported to ABA and deductive reasoning with uncertainty via the same notational standard (Section \ref{sec:other-frameworks}). We adopt ASPIC$^{+}$ notation for the sake of uniformity and because its syntax allows discerning among all different categories (e.g.\ strict and defeasible rules) that may come as merged in other formalisms. With this, we by no means claim that ASPIC$^{+}$ somehow encompasses other systems, as discussed in Section \ref{sec:other-frameworks} (see also the work by \textcite{arieli2021logic} for comparisons among different formalisms). Despite its interest and widespread applications, we leave DeLP aside in our treatment for the simple reason that it does not use Dung's argumentation frameworks as its abstract counterpart.

\par

As a key to establishing relative expressivity results between abstract and structured formalisms, we define a notion of
 \emph{expressive equivalence}, establishing the conditions for replicating each incomplete structured framework with an incomplete abstract one and vice-versa, where, roughly speaking, replicating means something as generating the ``same'' set of completions. The definition we provide in Section \ref{sec:def-expressivity} refines a previous one from the literature \cite{clar2021,mailly2024constrainedjournal} in order to make it applicable to both abstract and structured formalisms. A related notion of expressiveness has been studied in the context of abstract argumentation semantics via signatures of semantics \cite{dunne2015characteristics}, which compares semantics based on the sets of extensions they can generate. While our focus is on comparing formalisms through the sets of possible completions they can represent under uncertainty, both approaches share the common idea of characterising expressiveness in terms of sets of abstract outcomes.

%============================
\section{Abstract Argumentation with Qualitative Uncertainty}\label{sec:background}
%============================

Most of the systems of structured argumentation are endowed with an abstract representation, a \textit{derived} or \textit{associated graph}, usually called an \textit{argumentation framework}. In the derived graph, all plausible inferential constructions allowed by the structured framework are the nodes, whereas the different types of conflicts (usually determined by some contrariness function among propositions) are the arrows. Argumentation frameworks are defined as follows:

\begin{definition}[\citeauthor{dung1995acceptability}, \citeyear{dung1995acceptability}]
An \textbf{abstract argumentation framework} (AF) is a directed graph $\naf=\ang{\args,\defrel}$ where $\args$ is a set of \emph{arguments}
and $\defrel\subseteq \args\times \args$ is a \textit{defeat relation} among them. In what follows, we assume as given a \textbf{background set of arguments} (names) $\uni=\{a_1,...,a_n,...\}$ such that the domain $\args$ of any AF is a subset of $\uni$. 
\end{definition}
\noindent Intuitively, defeats are meant to encode all possible conflicts between two arguments, such as negating the other argument's premises, negating its conclusion, its inferential soundness, etc. \par

The key concept in abstract argumentation is that of \textbf{argumentation semantics}.
Roughly speaking, semantics encode the notion of justifiable conclusion from structured argumentation as the conclusion of some ``acceptable'' node.
Formally, a semantics $\sigma$ is a map $\af \mapsto \wp(\args)$,\footnote{We use $\wp$ throughout the paper to denote the power set operator.} where $\sigma(\af)$ provides all sets $E$ of arguments which constitute a ``justified position'' within $\af$, in other words, a set of jointly acceptable arguments (also called an \textit{extension} or a \textit{solution}). To capture the informal notion of acceptability, extensions are required to satisfy several specific constraints. As an example, consider \textbf{admissible extensions}, which are asked to satisfy conflict-freeness (i.e., $E \in \mathsf{admissible}(\af)$ implies that no $x,y \in E$ are such that $x\defrel y$) and self-defense (i.e., $E \in \mathsf{admissible}(\af)$ implies that for every $x \in \args$ such that $x\defrel y$ for some $y \in E$, there is a $z \in E$ such that $z \defrel x$). Here, we won't provide further details about argumentation semantics, since they are somehow orthogonal to our topic of interest in the present work.
{ For a systematic presentation of abstract argumentation semantics and their insights, the reader is referred to \textcite{baroni2018abstract}.}

%%%%===========================
\subsection{Abstract Argumentation Frameworks with Uncertain Arguments}\label{subsec:arg-IAFs}
%%%%===========================

In order to deal with uncertain arguments, we need to refine our definition of AFs by specifying two distinct sets of arguments. The definition runs as follows.

\begin{definition}[\citeauthor{arg-IAFs-original}, \citeyear{arg-IAFs-original}]
An \textbf{argument-incomplete abstract argumentation framework} (arg-IAF) is a tuple $\nargiaf=\ang{\args^{\fix},\args^{?},\defrel}$ where $\args^{\fix}$ and $\args^{?}$ are two pairwise disjoint sets of arguments and $\defrel\subseteq (\args^{\fix}\cup \args^{?})\times (\args^{\fix}\cup \args^{?})$.
\end{definition}

Compared to the standard notion of an AF, arguments are partitioned into two disjoint sets: the set $\args^{\fix}$ of fixed or certain arguments, i.e. arguments that are ``always there", and the set $\args^{?}$ is that of uncertain arguments. The role played by this distinction becomes clear with the definition of completion. We make use of the following notation: given a relation $S\subseteq B\times B$ and a set $B'\subseteq B$, the \textbf{restriction of $S$ to $B'$}, defined as $S\cap (B'\times B')$, and it is abbreviated as $S_{\restrict B'}$. 

\begin{definition}[\citeauthor{arg-IAFs-original}, \citeyear{arg-IAFs-original}]\label{def:AF_completions}
A \textbf{completion} of $\nargiaf=\ang{\args^{\fix},\args^{?},\defrel}$ is any AF $\ang{\args^{\ast},\defrel^{\ast}}$ such that:
\begin{itemize}
\item $\args^{\fix}\subseteq \args^{\ast}\subseteq \args^{\fix}\cup \args^{?}$.
\item $\defrel^{\ast}=\defrel_{\restrict  \args^{\ast}}$.
\end{itemize}
We denote by $\compfun(\nargiaf)$ the set of completions of $\nargiaf$. We adopt the same convention for the other formalisms that include completions in their definition.
\end{definition}

\noindent A completion thus represents a possible way uncertain arguments can be added to fixed ones to form an AF. It is important to note that defeats in an arg-IAF are instead all ``certain", and the second item of the definition guarantees that they are always present in a completion once their endpoints are there.

\begin{example}\label{ex:arg-IAF} The following figure depicts an arg-IAF, $\nargiaf=\ang{\args^{\fix}_0,\args^{?}_0,\defrel_0}$, where $\args^\fix_0=\{a\}$, $\args^?_0=\{b,c\}$ and $\defrel_0=\{\ang{b,a},\ang{c,a}\}$, where uncertain arguments are represented as dashed nodes:

\begin{center}
\begin{tikzpicture}[modal,world/.append style=
{minimum size=0.5cm}]
\node[fixarg] (a) [] {{$a$}};
\node[unarg] (b) [left=1cm of a]{{$b$}};
\node[unarg] (c) [right=1cm of a]{{$c$}};

\draw[->] (b) edge (a);
\draw[->] (c) edge (a);

\end{tikzpicture}
\end{center}
Its completions are the following AFs:

\begin{center}
\begin{tabular}{c|c|c| c}

\begin{tikzpicture}[modal,world/.append style=
{minimum size=0.5cm}]
\node[fixarg] (a) [] {{$a$}};
\node[fixarg] (b) [left=1cm of a]{{$b$}};
\node (nameaf) [above=0.2cm of b] {$\naf_0$};
\node[fixarg] (c) [right=1cm of a]{{$c$}};

\draw[->] (b) edge (a);
\draw[->] (c) edge (a);
\end{tikzpicture}
&

\begin{tikzpicture}[modal,world/.append style=
{minimum size=0.5cm}]
\node[fixarg] (a) [] {{$a$}};
\node[fixarg] (b) [left=1cm of a]{{$b$}};
\node (nameaf) [above=0.2cm of b] {$\naf_1$};

\draw[->] (b) edge (a);

\end{tikzpicture}
&

\begin{tikzpicture}[modal,world/.append style=
{minimum size=0.5cm}]
\node[fixarg] (a) [] {{$a$}};
\node (nameaf) [above=0.2cm of a] {$\naf_2$};
\node[fixarg] (c) [right=1cm of a]{{$c$}};

\draw[->] (c) edge (a);
\end{tikzpicture}
&

\begin{tikzpicture}[modal,world/.append style=
{minimum size=0.5cm}]
\node[fixarg] (a) [] {{$a$}};

\node (nameaf) [above=0.2cm of a] {$\naf_3$};
\end{tikzpicture}
\end{tabular}
\end{center}

\end{example}
\par \medskip

As mentioned in the introduction, simple arg-IAFs lack the power to discriminate among possible completions, therefore lacking a sufficient level of granularity for our present purposes. One way to achieve such granularity is to allow dependency conditions between (sets of) arguments. These conditions express things such as ``{at least one argument of} the set $Y$ is always present in a completion whenever {all arguments in} the set $X$ are present" (implicative dependencies) or else ``at least some argument from the set $X$ should be there" (disjunctive dependencies), or again ``not all arguments from the set $X$ can occur together" (NAND-dependencies). More formally, given an argument-incomplete abstract argumentation framework $\argprefix\niaf=\ang{\args^{\fix},\args^{?},\defrel}$, a \textbf{dependency} over $\argprefix\niaf$ is an expression of the form $\impor(X,Y)$, $\opor(X)$ or $\opnand (X)$ where $X$ and $Y$ are non-empty subsets of $\args^?$. The precise meaning of these dependencies is specified in the following definition.  

\begin{definition}[\citeauthor{fazzingaijcai21}, \citeyear{fazzingaijcai21}]\label{def:dep-arg-IAF} 
An \textbf{argument-incomplete abstract argumentation framework with dependencies} (dep-arg-IAF) 
is a tuple $\ndepargiaf=\ang{\args^{\fix},\args^{?},\defrel,\Delta}$ where $\ang{\args^{\fix},\args^{?},\defrel}$ is an argument-incomplete abstract argumentation framework and $\Delta$ is a set of dependencies over $\ang{\args^{\fix},\args^{?},\defrel}$.
A \textbf{completion} of $\ang{\args^{\fix},\args^{?},\defrel,\Delta}$ is any AF $\ang{\args^{\ast},\defrel^{\ast}}$ such that:
\begin{itemize}
\item $\ang{\args^{\ast},\defrel^{\ast}}$ is a completion of $\ang{\args^{\fix},\args^{?},\defrel}$; and
\item for all $\delta \in \Delta$:

\begin{itemize}
    \item If $\delta=\impor(X,Y)$, then $X\subseteq \args^\ast$ implies $Y \cap \args^\ast\neq \varnothing$.
    \item If $\delta=\opor(X)$, then $X\cap \args^\ast\neq \varnothing$.
     \item If $\delta=\opnand(X)$, then $X \cap \args^\ast \subset X$.
\end{itemize}
\end{itemize}

An \textbf{implicative argument-incomplete abstract argumentation framework} (imp-arg-IAF) is a tuple $\nimpargiaf=\ang{\args^{\fix},\args^{?},\defrel,\Delta}$ where $\Delta$ only contains dependencies of the kind $\impor(X,Y)$ where $Y$ is a singleton set. Given a dependency of this kind $\impor(X,\{y\})$, we simplify notation and write $\opimp(X,y)$.
\end{definition}

\begin{example} If we augment the arg-IAF of Example \ref{ex:arg-IAF} with $\Delta_0=\{\opimp (\{b\},c)\}$, then completion $\naf_1$ is excluded (i.e., $\naf_1 \notin \compfun(\ang{\args^{\fix}_0,\args^{?}_0,\defrel_0,\Delta_0})$. Analogously, if we take $\Delta_1$ to be $\{\opor(\{b,c\})\}$ then $\compfun(\ang{\args^{\fix}_0,\args^{?}_0,\defrel_0,\Delta_1})=\compfun(\ang{\args^{\fix}_0,\args^{?}_0,\defrel_0})\setminus \{\naf_3\}$.
    
\end{example}
The following result shows that dep-arg-IAFs are sufficiently powerful to isolate any possible set of completions of the underlying arg-IAF.

\begin{proposition}\label{prop:dep-arg-iaf}
Let $\nargiaf = \ang{\args^{\fix},\args^{?},\defrel}$ be an arg-IAF and $X \subseteq \compfun(\nargiaf)$ a subset of its possible completions. There is always a set $\Delta$ of dependencies such that $\compfun(\ang{\args^{\fix},\args^{?},\defrel,\Delta}) = X$.
\end{proposition}

\begin{proof} The argument provided by \textcite{fazzingaijcai21} for finite arg-IAFs can be generalised to the infinite case. If we look at $\args^?$ as a set of propositional variables, then dependencies over $\ang{\args^{\fix},\args^{?},\defrel}$ can be translated into formulas of the infinitary propositional language $\lanset(\args^?)$ given by the BNF:

\begin{center}
    $\varphi::= x \mid \bigwedge X \to \bigvee Y \mid \bigvee X \mid \lnot \bigwedge X$ where $x \in \args^?$ and $X,Y \in \wp(\args^?)\setminus\{\varnothing\}$. 
\end{center}
    In this picture, completions of $\ang{\args^{\fix},\args^{?},\defrel}$ correspond to propositional valuations over $\args^?$ where variables corresponding to arguments in the completions evaluate to true and other variables are set to false.
    Further, completions of $\ang{\args^{\fix},\args^{?},\defrel, \Delta}$ are those among the just described set of valuations that satisfy all formulas from $\Delta$. \\
    Importantly, any set of valuations over $\args^?$ can be encoded as a formula $\psi$ over $\args^?$ in conjunctive normal form (possibly using an infinitary language with infinite conjunctions and disjunctions if the set $\args^?$ is infinite), where every conjunct corresponds to (the negation of) a valuation not in the set. 
    This formula can then be translated into an equivalent set of formulas of  $\Gamma\subseteq \lanset(\args^?)$ (i.e., to a set of dependencies!), reasoning by cases over the conjuncts of $\psi$. If a conjunct $X$ only contains positive literals $x_1,...,x_n,...$ then, we put $\bigvee X$ into $\Gamma$. If it only contains negative literals $\lnot x_1,...,\lnot x_n,...$, then we add $\lnot \bigwedge \{x_1,...,x_n,...\}$ to $\Gamma$. Finally, if $X$ contains both positive and negative literals $\{x_1,\dots, x_n, \dots, \lnot y_1 \dots \lnot y_k,\dots\}$, then we add $\bigwedge \{y_1,\dots,y_k,\dots\} \to \bigvee \{x_1,\dots,x_n,\dots\}$ to $\Gamma$.
\end{proof}

\begin{remark}
 Our definition of an argument-incomplete argumentation framework with dependencies differs from the one provided by \textcite{fazzingaijcai21} because we exclude $\mathsf{CHOICE}$-dependencies. However, as explained by these authors \cite[p. 191]{fazzingaijcai21} and generalised in the previous proposition, this simplification does not affect the expressive power of dep-arg-IAFs, which will be our main focus of interest here. Moreover, in this regard, it is equivalent to using \textit{constrained argument-incomplete argumentation frameworks}, a subclass of the more general constrained incomplete argumentation frameworks \cite{clar2021,jlc,maillyciafs,mailly2024constrainedjournal}. In these structures, dependencies are expressed through a Boolean formula in a language that contains a propositional variable for each argument in $\args^?$. 
\end{remark}

%======================
\subsection{Relative Expressivity}\label{sec:def-expressivity}
%======================

The formalisms we have just presented can be compared with respect to their ability to express different sets of completions. To this end, and to deal with structured (i.e., non-abstract) formalisms, we refine a notion of expressivity previously introduced in the literature \cite{clar2021,mailly2024constrainedjournal}.

\begin{definition}\label{def:expressivity}
    Given two sets of AFs (i.e., of completions) $S=\{\ang{\args_1,\defrel_1},...\}$ and $S'=\{\ang{\args_1',\defrel_1'},...\}$ we say that they are \textbf{equivalent} (in symbols, $S\approxeq S'$) iff there is a bijective function $i:\bigcup_{\ang{\args,\defrel}\in S}\args \to \bigcup_{\ang{\args',\defrel'}\in S'}\args'$ such that $\{\ang{i[\args],i[\defrel]} \mid \ang{\args,\defrel}\in S \}=S'$, where $i[\args]=\{i(x)\mid x \in \args\}$ and $i[\defrel]=\{\ang{i(x),i(y)}\mid \ang{x,y}\in \defrel\}$. Note that if such a function exists, then it is a one-to-one isomorphism among the elements of $S$ and $S'$.\par 

Given two classes of argumentative formalisms with qualitative uncertainty $\mathcal{X}$ and $\mathcal{Y}$  (e.g., arg-IAFs, imp-arg-IAFs, dep-arg-IAFs, etc),\footnote{ Formally, given a background set of arguments $\uni$, an argumentative formalism with qualitative uncertainty is any formalism $X$ equipped with a function $\compfun(\cdot)$ that assigns to $X$ its set of completions $\compfun(X)$, that is, a set of directed graphs over $\uni$.} we say that \textbf{$\mathcal{X}$ is at least as expressive as $\mathcal{Y}$} (in symbols, $\mathcal{Y} \lexpressive \mathcal{X}$) iff for all $Y \in \mathcal{Y}$ there is an $X \in \mathcal{X}$ such that $\compfun(X) \approxeq \compfun(Y)$. If the function $i$ is clear enough from context, we omit it and identify each $x$ with $i(x)$.
\end{definition}

\begin{fact}\label{fact:lespressive-is-a-preorder} $\lexpressive$ is reflexive and transitive. \end{fact}

 \begin{remark}\label{rmk:abstract-expressivity}
     When dealing with two classes of abstract formalisms $\mathcal{X}$ and $\mathcal{Y}$ (whose arguments belong to $\uni$), one can show that the notion of equivalence boils down to identity as far as expressivity is concerned. That is to say, $\mathcal{Y} \lexpressive \mathcal{X}$ holds iff for all $Y \in \mathcal{Y}$ there is an $X \in \mathcal{X}$ such that $\compfun(X) =\compfun(Y)$. { This is precisely the notion of expressivity used for abstract formalisms that we took as a departure point \cite{mailly2022yes,mailly2024constrainedjournal,jlc,fazzingaijcai21,fazzingakr21}. Hence, our proposal is a conservative extension of a standard notion.}
 \end{remark}

 \begin{fact} $\text{arg-IAFs }\lexpressive \text{ dep-arg-IAFs}$. To see this, note that  $\compfun(\ang{\args^\fix,\args^?,\defrel})=\compfun(\ang{\args^\fix,\args^?,\defrel,\varnothing})$ (where $\varnothing$ is a set of dependencies). 
 \end{fact}

 \begin{fact}\label{fact:imp-less-expressive-dep} $\text{imp-arg-IAFs }\not \lexpressive \text{ dep-arg-IAFs}$. To see this, take $\ndepargiaf=\ang{\args^\fix,\args^?,\defrel,\Delta}$ with $\args^\fix=\varnothing$, $\args^?=\{a,b\}$ and $\Delta=\opnand(\{a,b\})$. It is easy to see, by checking all possible imp-arg-IAFs with $a$ and $b$ as arguments, that there is none with the same set of completions.
 \end{fact}

\begin{remark}[On weaker notions of equivalence]\label{rmrk:isomorphism} One might wonder whether our notion of equivalence is just too strong. For instance, it might seem at first sight that it is enough to require the existence of a bijection $i:S \to S'$ such that for each $\af \in S$, $\af$ is isomorphic to $i(\af)$. However, this weaker notion finds some counterintuitive positive cases. For instance, the following $S$ and $S'$ are equivalent under this alternative notion:

\begin{center}
    
\begin{tikzpicture}[->,>=stealth,shorten >=1pt,auto,node distance=1.4cm,
                thick,main node/.style={circle,draw,font=\bfseries},uncertain/.style={rectangle,draw,dashed,font=\bfseries}]

\node[arg] (a) {$a$};
\node[arg] (b) [right of= a] {$b$};
\path[->] (a) edge (b);
\node (ad1) [left of= a] {$\ang{\args_1,\defrel_1}$};

\node[arg] (a1) [below of= a] {$a$};
\node (ad1) [left of= a1] {$\ang{\args_2,\defrel_2}$};

\node[arg] (b1) [right of= a1] {$b$};
\path[->] (b1) edge (a1);

\node (firstset) [draw, fit=(a) (b1),label=above:$S$] {};

% SECOND SET
\node[arg] (a2) [right=2cm of b] {$a$};
\node[arg] (b2) [right of= a2] {$b$};
\path[->] (a2) edge (b2);
\node (ad1p) [right of= b2] {$\ang{\args_1',\defrel_1'}$};

\node[arg] (c) [below of= a2] {$c$};
\node[arg] (d) [right of= c] {$d$};
\path[->] (d) edge (c);
\node (ad2p) [right of= d] {$\ang{\args_2',\defrel_2'}$};

\node (secondset) [draw, fit=(a2) (d),label=above:$S'$] {};

%%%%%%

% horizontal divider in first set
\draw[-] ($(firstset.north west)!0.5!(firstset.south west)$)
   -- ($(firstset.north east)!0.5!(firstset.south east)$);

% horizontal divider in second set
\draw[-] ($(secondset.north west)!0.5!(secondset.south west)$)
   -- ($(secondset.north east)!0.5!(secondset.south east)$);
\end{tikzpicture}
    \end{center}
\noindent
    Note that $b$ gets accepted in one of the AFs of $S$ under any of the semantics defined by Dung \cite{dung1995acceptability}; but it is not accepted in any of the AFs of $S'$. Intuitively, sets of frameworks where one accepts different arguments cannot be argumentatively equivalent.  If we transpose the previous example to structured argumentation frameworks, we could have equivalent frameworks in which non-equivalent propositions are accepted/believed, which is a clear undesired consequence. { To avoid this, our notion of equivalence keeps track of arguments' identity throughout the equivalent sets of completion.}
\end{remark}

%%%%%%==================
\section{ASPIC$^{+}$ Frameworks with Uncertain Components}\label{sec:aspic-with-uncertainty}
%%%%%%==================

We now move to ASPIC$^{+}$ as our first framework for structured argumentation. Here, the preliminary question is how to define the notion of an incomplete ASPIC$^{+}$ framework. We have essentially three options, already mentioned in the literature: uncertainty generated by uncertain inference rules \cite{baumeister2021acceptance} (studied more in detail by \citeauthor{ai32023}, \citeyear{ai32023}), by uncertain premises \cite{odekerken2023argumentative} (later developed by \citeauthor{odekerken2025argumentative}, \citeyear{odekerken2025argumentative}), or by uncertain preference profiles \cite{baumeister2021acceptance}. 
Here, we focus on the first two options, providing the formal details of their definitions in Sections \ref{subsec:uncertainrules} and \ref{subsec:uncertainpremises}. We then proceed to the main task of comparing their expressive power, as defined in Section \ref{sec:def-expressivity}, with respect to the abstract formalisms reviewed above.

%%%%===========================
\subsection{ASPIC$^{+}$ in a Nutshell}
%%%%===========================

Here we provide the necessary ASPIC$^{+}$ background in a self-contained but compact way. For the motivation behind the definitions, as well as discussion of alternative notions, the reader is referred to \cite{modgil2013aspic,modgil2014tutorial}. As for AFs in abstract argumentation, the central notion here is that of a \textit{structured argumentation framework} (SAF) and its associated, derived AF. However, we need a number of previous notions before getting there.

\begin{definition}[\citeauthor{modgil2013aspic}, \citeyear{modgil2013aspic}]\label{def:aspic-theory}
    An \textbf{argumentation theory} is a tuple $\at=\ang{\lanset, \negfun, \rulset, \namefun,\kb}$ where:
\begin{itemize}
\item $\lanset$ is a formal language.
\item $\negfun:\lanset \to \wp(\lanset)$ is a contrariness function. We say that:
\begin{itemize}
\item $\varphi$ is a contrary of $\psi$ iff $\varphi \in \overline{\psi}$ but $\psi \notin \overline{\varphi}$.
\item $\varphi$ is a contradictory of $\psi$ iff $\varphi \in \overline{\psi}$ and $\psi \in \overline{\varphi}$.
\end{itemize}
It is assumed that each $\varphi \in \lanset$ has at least one contradictory, denoted $-\varphi$.

\item $\rulset=\rulset_s \cup \rulset_d$ with $\rulset_s \cap \rulset_d =\varnothing$ is a set of inference rules (pairs of finite sets of formulas and formulas, $\rulset \subseteq \wp_{fin}(\lanset)\times\lanset$).\footnote{
Given a set $S$, we denote by $\wp_{fin}(S)$ the set of all its finite subsets.} $\rulset_s$ represents strict rules while $\rulset_d$ represents defeasible rules.

\item $\namefun:\rulset_d \to \lanset$ is a partial naming function for defeasible rules.
\item $\kb \subseteq \lanset$ is a knowledge base which comes split into two disjoint subsets $\kb_n$ (axioms) and $\kb_p$ (ordinary premises).
\end{itemize}
\end{definition}

The central notion of argument is then defined as follows:
\begin{definition}[\citeauthor{modgil2013aspic}, \citeyear{modgil2013aspic}]
The set of \textbf{arguments of a given argumentation theory} $\at=\unravelat$, denoted $\args(\unravelat)$, is defined inductively. Together with the notion of argument, we define some auxiliary functions: $\sub(\cdot)$ (returns the \textbf{subarguments} of an argument), $\prem(\cdot)$ (returns the \textbf{premises} of an argument), $\conc(\cdot)$ (returns the \textbf{conclusion} of an argument), and $\ftoprule(\cdot)$ (returns the \textbf{last rule} employed in the construction of an argument). We establish that $A \in \args(\unravelat)$ iff $A$ is any finite expression built by the application of the following conditions:

\begin{itemize}
\item $A=[\varphi]$ if $\varphi \in \kb$, with 
\begin{itemize}
    \item 
$\prem(A)=\conc(A)=\{\varphi\}$, 
\item 
$\sub(A)=\{[\varphi]\}$, and 
\item $\ftoprule(A)$ is left undefined.

\end{itemize}
\item $A=[A_1,...,A_n \sto \varphi]$ (with $n\geq 0$) if $A_1,...,A_n$ are arguments and $\ang{\{\conc(A_1),...,\conc(A_n)\},\varphi} \in \rulset_s$, with 

\begin{itemize}
\item $\prem(A)=\prem(A_1)\cup...\cup\prem(A_n)$, 

\item $\conc(A)=\varphi$, 

\item $\sub(A)=\{A\}\cup \sub(A_1)\cup...\cup\,\sub(A_n)$, 

\item $\ftoprule(A)=\ang{\{\conc(A_1),...,\conc(A_n)\},\varphi}$.

\end{itemize}
\item $A=[A_1,...,A_n \Rightarrow \varphi]$ (with $n\geq 0$) if $A_1,...,A_n$ are arguments and $\ang{\{\conc(A_1),...,\conc(A_n)\},\varphi} \in \rulset_d$, with:

\begin{itemize}
\item $\prem(A)=\prem(A_1)\cup...\cup\prem(A_n)$, 

\item $\conc(A)=\varphi$, 

\item $\sub(A)=\{A\}\cup \sub(A_1)\cup...\cup\,\sub(A_n)$, 

\item $\ftoprule(A)=\ang{\{\conc(A_1),...,\conc(A_n)\},\varphi}$.

\end{itemize}
\end{itemize}

 \noindent We omit squared brackets whenever no ambiguity arises. Note that arguments do not depend on $\contfun$ nor on $\namefun$, so we sometimes abbreviate $\args(\unravelat)$ as $\args(\lanset,\rulset,\kb)$.  Given $A \in \args (\at)$ we define the \textbf{rules of $A$} as $\rulset(A)=\{\ftoprule(B)\mid B \in \sub(A)\}$.  \par \medskip

   \begin{remark}[Rules and arguments without premises]\label{remark:premise-less-rules} The previous definitions allow for (i) rules with an empty set of formulas on the left-hand side (i.e., it is possible that $\ang{\{\},\varphi} \in \rulset$); and consequently (ii) arguments with an empty set of premises ($\Rightarrow \varphi$ and $\sto\varphi$, but also more complex ones like $\sto \varphi, \Rightarrow \psi \Rightarrow \delta$). If an argument has no premises, we say it is \textbf{premiseless}. We say that an argument is \textbf{simple} iff it is either a single formula (a premise) or an argument of form $\Rrightarrow\varphi$ with $\Rrightarrow\in \{\sto,\Rightarrow\}$.
     \end{remark}
     
      \begin{remark}[Excluding infinite arguments] The original definition of ASPIC$^+$ arguments allows for infinite arguments \cite{modgil2013aspic}: arguments which are rooted in a conclusion but keep growing infinitely without ``finding'' a set of premises. These abnormal arguments are later excluded from the definition of associated AF in the ASPIC$^{+}$  literature. We exclude them from the very definition of argument for presentational purposes.         
     \end{remark}
     
\end{definition}
Next, one needs to codify a notion of conflict among arguments.
\begin{definition}[\citeauthor{modgil2013aspic}, \citeyear{modgil2013aspic}]
Given an argumentation theory $\at=\unravelat$, and two arguments $A,B\in \args(\at)$, we say that \textbf{$A$ attacks $B$} iff $A$ undermines, rebuts or undercuts $B$, where: 

\begin{itemize}
\item $A$ \textbf{undermines} $B$ (on $B'$) iff $\conc(A)\in \overline{\varphi}$ for some $B'=\varphi \in \prem(B)$ such that $\varphi \in \kb_p$. {In that case we say that $A$ contrary-undermines $B$ if $\conc(A)$ is a contrary of $\varphi$.}
\item $A$ \textbf{rebuts} $B$ (on $B'$) iff $\conc(A)\in \overline{\varphi}$ for some $B' \in \sub(B)$ of the form $B_1',....,B_n'\Rightarrow\varphi$. {In that case we say that $A$ contrary-rebuts $B$ if $\conc(A)$ is a contrary of $\varphi$.}

\item $A$ \textbf{undercuts} $B$ (on $B'$) iff $\conc(A)\in \overline{\namefun(\ftoprule(B'))}$ for some $B' \in \sub(B)$ with $\ftoprule(B')\in \rulset_{d}$.
\end{itemize}
\end{definition}

A simple conflict is sometimes not enough to decide whether an argument wins over another. To this end, ASPIC$^{+}$ adds a preference relation among arguments, which is integrated in the formalism as follows:

\begin{definition}[\citeauthor{modgil2013aspic}, \citeyear{modgil2013aspic}]
A \textbf{structured argumentation framework} (SAF) is a tuple $\nsaf=\unravelsaf$ where $\at=\unravelat$ is an argumentation theory and $\preceq$ is a preferential ordering (usually a partial preorder) relation among arguments $\preceq \subseteq \args(\at)\times \args(\at)$ (its strict counter-part of $\preceq$, denoted $\prec$, is defined as usual: $\prec=\preceq\setminus \preceq^{-1}$).
\end{definition}
In many contexts and applications, preferences are usually defined among formulas or rules and then lifted to arguments. Different lifting principles that respect rationality postulates (i.e., desirable properties about the framework output) have been studied (see \cite{modgil2013aspic}).
The preference relation then allows us to define a full-blown notion of defeat.

\begin{definition}[\citeauthor{modgil2013aspic}, \citeyear{modgil2013aspic}]
Given $\nsaf=\unravelsaf$, and $A,B\in \args(\unravelat)$, we say that $A$ \textbf{defeats} $B$ iff:
\begin{itemize}
 \item[(i)]  $A$ undercuts/contrary-rebuts/contrary-undermines $B$;  or 
 \item[(ii)] $A$ undermines/rebuts $B$ (on $B'$) and $A\not \prec B'$. 
\end{itemize}
The set of all defeats for a given $\nsaf$ is denoted $\defrel(\nsaf)$. We instead equate the set of arguments $\args(\nsaf)$ of a given $\nsaf$ with those of its underlying argumentation theory, i.e. $\args(\nsaf) = \args(\at)$. In general, given $\nsaf=\unravelsaf$, we use $\lanset(\nsaf)$ to denote $\lanset$ and apply the same convention for the rest of the components (including the non-primitive components $\args(\nsaf)$ and $\rel(\nsaf)$). 
\end{definition}

The following definition provides the key notion that lifts structured argumentation frameworks to the abstract level by identifying their associated graph.
\begin{definition}[\citeauthor{modgil2013aspic}, \citeyear{modgil2013aspic}]
   Let $\nsaf=\unravelsaf$ be given, the \textbf{abstract argumentation framework associated to $\nsaf$} is defined as $\daf(\nsaf)=\ang{\args(\nsaf),\defrel(\nsaf)}$. 
\end{definition}

\begin{example}\label{example:saf} Consider $\nsaf_0=\indexedsaf{}$, where:
\begin{itemize}
    \item $\lanset$ is the language of propositional logic; 
    \item $\overline{\cdot}$ is given by classical negation (i.e., $\varphi \in \overline{\psi}$ iff $\varphi=\lnot \psi$ or $\psi=\lnot \varphi$);
    \item $\rulset_s=\{\ang{\{u\},\lnot s}\}$;
    \item $\rulset_d=\{\ang{\{p\},q}, \ang{\{w\},r}, \ang{\{s\},\lnot r} \}$;
    \item  $\namefun$ is only defined for $\namefun(\ang{\{p\},q})=r$;
    \item $\kb_n=\{p,u\}$;
    \item $\kb_p=\{s, w\}$;
    \item $\preceq =\{\ang{ s \Rightarrow \lnot r, w \Rightarrow r }\} $.
\end{itemize}

The associated AF looks as follows, where each box is an argument and arrows represent the defeat relation:

\begin{center}
\begin{tikzpicture}[->,>=stealth,shorten >=1pt,auto,node distance=1.4cm,
                thick,main node/.style={circle,draw,font=\bfseries},uncertain/.style={rectangle,draw,dashed,font=\bfseries}]

\node[sarg] (s) {$s$};
\node[sarg] (snr) [right=0.5cm of s] {$s\Rightarrow \lnot r$};

\node[sarg] (pq) [right=2cm of snr] {$p \Rightarrow q$};
\node[sarg] (p) [below=0.1cm  of pq] {$p$};

\node[sarg] (uns) [above of=s] {$u\sto \lnot s$};
\node[sarg] (u) [above=0.1cm of uns] {$u$};

\node[sarg] (wr) [below of=s] {$w\Rightarrow r$};
\node[sarg] (w) [below=0.1cm  of wr] {$w$};

\path[->] (snr) edge (pq);
\path[->] (uns) edge (snr);
\path[->] (uns) edge (s);
\path[->] (wr) edge (snr);
\end{tikzpicture}
\end{center}
    
\end{example}

%%%%%%==================
\subsection{Uncertain Inference Rules}\label{subsec:uncertainrules}
%%%%%%================== 
 
 To define incomplete structured frameworks, we first consider the set of rules of a given argumentation system as a source of uncertainty, following the IAFs spirit. A definition along these lines was first suggested by \textcite{baumeister2021acceptance} and formally specified by \textcite{ai32023}.
 
 \begin{definition}[\citeauthor{ai32023}, \citeyear{ai32023}]\label{def:rul-ISAFs}
     A \textbf{rule-incomplete structured argumentation framework} (rul-ISAF) is a tuple $\rulprefix\nisaf=\unravelsaf$, where every component is just as in a SAF except for the set of rules $\rulset$, which is split into four pairwise disjoint subsets $\rulset=\rulset^{F}_s\cup\rulset_{s}^{?}\cup\rulset^{F}_{d}\cup\rulset^{?}_{d}$, representing respectively certain strict rules, uncertain strict rules, certain defeasible rules and uncertain defeasible rules.  We define $\rulset^{F}=\rulset^{F}_s\cup \rulset^{F}_{d}$ (the set of certain rules) and $\rulset^{?}=\rulset^{?}_s\cup \rulset^{?}_{d}$ (the set of uncertain rules). 
 \end{definition}

 \begin{remark}[On uncertain strict rules] The notion of uncertain strict rule (elements of $\rulset_{s}^{?}$) might seem contradictory at first sight. However, this kind of rule makes sense if we think of a rul-ISAF as a model of what an agent thinks of its opponent's (or simply other agents') argumentative situation. Hence, 
 a rule $\rulset_{s}^{?}$ can be understood as a strict rule whose application by the opponent is uncertain from the agent's perspective.
 \end{remark}

For a given rul-ISAF, we can define its completions analogously to Definition \ref{def:AF_completions}, here considering the addition of all possible sets of uncertain rules (instead of abstract arguments):
 
 \begin{definition}[\citeauthor{ai32023}, \citeyear{ai32023}\label{def:rul-isaf-completions}]
     
A \textbf{rule-completion of $\rulprefix\nisaf$} is any 
 $\nsaf^{\ast}=\ang{\lanset, \negfun, \rulset^{\ast}, {\namefun^\ast}, \kb, \preceq^{\ast}}$ such that:
\begin{itemize}
\item $\rulset^{^\ast}=\rulset_{s}^{\ast}\cup \rulset_{d}^{\ast}$ is such that:
\begin{itemize}
\item $\rulset^{\fix}_{s}\subseteq \rulset_{s}^{\ast}\subseteq (\rulset^{\fix}_{s} \cup \rulset_{s}^{?})$; and
\item $\rulset^{\fix}_{d}\subseteq \rulset_{d}^{\ast}\subseteq (\rulset^{\fix}_{d} \cup \rulset_{d}^{?})$.
\end{itemize}

\item  $\namefun^\ast= \namefun \cap (\rulset^\ast\times \lanset)$.

\item $\preceq^{\ast}=\preceq\cap(\args^\ast \times \args^\ast)$ where $\args^\ast$ is the set of arguments generated using $\rulset^\ast$ (i.e., $\args^\ast=\args(\lanset,\rulset^\ast,\kb)$).
\end{itemize}

We denote by $\rulprefix\compfun(\nisaf)$ the set of rule-completions of $\rulprefix\nisaf$.

 \end{definition}
Two specific rule-completions will be used for proving our central results, namely:
\begin{itemize}
\item $\nsaf^{\fix}$ is the rule-completion whose set of rules is $\rulset^{\fix}$.
\item $\nsaf^{max}$ is just $\rulprefix\nisaf$ (i.e., the rule-completion that uses both certain and uncertain rules $\rulset^\fix \cup \rulset^?$).
\end{itemize}
$\nsaf^{\fix}$ is the minimal completion, where all uncertain rules are discarded, while $\nsaf^{max}$ is the maximal one, where all rules are accepted. 
Let $\nsaf^\ast \in \rulprefix\compfun(\nisaf)$ and let $X \in \args(\nsaf^\ast)$, we define the set of uncertain (resp.\ certain) rules of argument $X$ as $\rulset^{?}(X)=\rulset(X)\cap \rulset^{?}$ (resp.\ $\rulset^{\fix}(X)=\rulset(X)\cap \rulset^{\fix}$). This function is lifted to sets of arguments by setting $\rulset^?(\Gamma)=\bigcup_{Y \in 
\Gamma}\rulset^?(Y)$.

\par
\medskip 

To compare the expressivity of incomplete formalisms, we again need to lift them at the abstract level by defining the abstract completions of a given $\rulprefix\nisaf$. These simply correspond to the abstract argumentation frameworks associated with its rule-completions.

\begin{definition}[\citeauthor{ai32023}, \citeyear{ai32023}]
    The set of (abstract) \textbf{completions} of a given $\rulprefix\nisaf$ is:
\begin{align*}
\compfun(\rulprefix\nisaf)=&\{\daf(\nsaf^{\ast})\mid  \nsaf^{\ast}\in \rulprefix\compfun(\rulprefix\nisaf)\}\text{.}    
\end{align*}
\end{definition}

\begin{example}
    Consider $\rulprefix\nisaf_0=\unravelsaf$, where every component is as in the SAF of Example \ref{example:saf} except for $\rulset$, which is defined as follows:
\begin{itemize}
    \item $\rulset^\fix_s=\varnothing$;
        \item $\rulset^?_s=\{\ang{\{u\},\lnot s}\}$;
    \item $\rulset^\fix_d=\{\ang{\{w\},r}, \ang{\{s\},\lnot r} \}$; and
        \item $\rulset^?_d=\{\ang{\{p\},q}\}$.
\end{itemize}
This rul-ISAF has four rule-completions (one for each subset of $\rulset^?$) and four abstract completions. These are represented as follows: 
\par \smallskip
\scalebox{0.95}{
\begin{tabular}{c | c}

% FIRST COMPLETION rul-ISAF example 
     \begin{tikzpicture}[->,>=stealth,shorten >=1pt,auto,node distance=1.4cm,
                thick,main node/.style={circle,draw,font=\bfseries},uncertain/.style={rectangle,draw,dashed,font=\bfseries}]

\node[sarg] (s) {$s$};
\node[sarg] (snr) [right=0.5cm of s] {$s\Rightarrow \lnot r$};

\node[sarg] (p) [below=0.1cm  of pq] {$p$};

\node[sarg] (u) [above=0.1cm of uns] {$u$};

\node[sarg] (wr) [below of=s] {$w\Rightarrow r$};
\node[sarg] (w) [below=0.1cm  of wr] {$w$};

\path[->] (wr) edge (snr);
\end{tikzpicture}
\qquad
& 
\qquad

% SECOND COMPLETION rul-ISAF example
\begin{tikzpicture}[->,>=stealth,shorten >=1pt,auto,node distance=1.4cm,
                thick,main node/.style={circle,draw,font=\bfseries},uncertain/.style={rectangle,draw,dashed,font=\bfseries}]

\node[sarg] (s) {$s$};
\node[sarg] (snr) [right=0.5cm of s] {$s\Rightarrow \lnot r$};

\node[sarg] (pq) [right=2cm of snr] {$p \Rightarrow q$};
\node[sarg] (p) [below=0.1cm  of pq] {$p$};

\node[sarg] (u) [above=0.1cm of uns] {$u$};

\node[sarg] (wr) [below of=s] {$w\Rightarrow r$};
\node[sarg] (w) [below=0.1cm  of wr] {$w$};

\path[->] (snr) edge (pq);

\path[->] (wr) edge (snr);
\end{tikzpicture}

\\
\hline

& \\

% THIRD COMPLETION rul-ISAF example
\begin{tikzpicture}[->,>=stealth,shorten >=1pt,auto,node distance=1.4cm,
                thick,main node/.style={circle,draw,font=\bfseries},uncertain/.style={rectangle,draw,dashed,font=\bfseries}]

\node[sarg] (s) {$s$};
\node[sarg] (snr) [right=0.5cm of s] {$s\Rightarrow \lnot r$};

%\node[sarg] (pq) [right=2cm of snr] {$p \Rightarrow q$};
\node[sarg] (p) [below=0.1cm  of pq] {$p$};

\node[sarg] (uns) [above of=s] {$u\sto \lnot s$};
\node[sarg] (u) [above=0.1cm of uns] {$u$};

\node[sarg] (wr) [below of=s] {$w\Rightarrow r$};
\node[sarg] (w) [below=0.1cm  of wr] {$w$};

%\path[->] (snr) edge (pq);
\path[->] (uns) edge (snr);
\path[->] (uns) edge (s);
\path[->] (wr) edge (snr);
\end{tikzpicture}
     & 

     \begin{tikzpicture}[->,>=stealth,shorten >=1pt,auto,node distance=1.4cm,
                thick,main node/.style={circle,draw,font=\bfseries},uncertain/.style={rectangle,draw,dashed,font=\bfseries}]

\node[sarg] (s) {$s$};
\node[sarg] (snr) [right=0.5cm of s] {$s\Rightarrow \lnot r$};

\node[sarg] (pq) [right=2cm of snr] {$p \Rightarrow q$};
\node[sarg] (p) [below=0.1cm  of pq] {$p$};

\node[sarg] (uns) [above of=s] {$u\sto \lnot s$};
\node[sarg] (u) [above=0.1cm of uns] {$u$};

\node[sarg] (wr) [below of=s] {$w\Rightarrow r$};
\node[sarg] (w) [below=0.1cm  of wr] {$w$};

\path[->] (snr) edge (pq);
\path[->] (uns) edge (snr);
\path[->] (uns) edge (s);
\path[->] (wr) edge (snr);
\end{tikzpicture}

\end{tabular}
}
\end{example}

\begin{remark}
    The definition of rul-ISAFs and rule-completion assumes that for each rule $R \in \rulset$, it is known whether $R$ is strict or defeasible. The only thing that might not be known is whether $R$ is applicable. 
\end{remark}

The first observation about rul-ISAFs is that they generate only argument-incompleteness and not defeat-incompleteness at the abstract level.

\begin{proposition}\label{prop:rul-ISAFs-arg-incompleteness}
   Let $\rulprefix\nisaf$ be a rul-ISAF, and let $\ang{\args_1,\defrel_1},\ang{\args_2,\defrel_2}\in \compfun(\rulprefix\nisaf)$. Then, for every $X,Y \in \args_1\cap \args_2$, $\ang{X,Y} \in \defrel_1$ iff $\ang{X,Y}\in \defrel_2$.
\end{proposition}

\begin{proof}
    By definition of completion, we have that $\ang{\args_1,\defrel_1}$ and $\ang{\args_2,\defrel_2}$  are generated by the rule-completions $\nsaf_1$ and $\nsaf_2$ respectively. Since $X,Y \in \args_1\cap \args_2$, we have that both rule-completions share at least the uncertain rules that take part in the construction of $X$ and $Y$. Now, suppose that $\ang{X,Y} \in \defrel_1$. The latter is the case iff (i) $X$ undercuts/contrary-undermines/contrary-rebuts $Y$ (with respect to $\nsaf_1$), or (ii) $X$ undermines/rebuts $Y$ (with respect to $\nsaf_1$) and $X\not \prec_1 Y'$. On the one hand, we have that $X$ undercuts/contrary-undermines/contrary-rebuts $Y$ (with respect to $\nsaf_1$) iff $X$ undercuts/contrary-undermines/contrary-rebuts $Y$ (with respect to $\nsaf_2$), because the definition of undercutting/contrary-undermining/contrary-rebuttal only involves (parts of) components of $\nsaf_1$ that are also present in $\nsaf_2$ (namely, $\overline{\cdot}_1$, the rules of $Y$ and $\namefun_1$). Similarly, we have that $X$ undermines/rebuts $Y$ on $Y'$ (with respect to $\nsaf_1$) and $X\not \prec_1 Y'$ iff $X$ undermines/rebuts $Y$ on $Y'$ (with respect to $\nsaf_2$) and $X\not \prec_2 Y'$. To see this, note that every element taking part in the definition of undermining and rebuttal for $X$ and $Y$ are shared by $\nsaf_1$ and $\nsaf_2$, and we have that $X\not \prec_1 Y'$ iff $X\not \prec_2 Y'$ holds by the definition of completions for rul-ISAFs.
 \end{proof}

Our first central result shows that rul-ISAFs are at least as expressive as arg-IAFs from Section \ref{subsec:arg-IAFs}. Its proof proceeds by showing how to associate, to any arg-IAF, a corresponding rul-ISAF with an equivalent set of completions.

\begin{theorem}\label{thm:r-ISAF-more-than-a-IAFs}
           $\text{arg-IAFs} \lexpressive \text{rul-ISAFs}$.
\end{theorem}

\begin{proof}
    Let $\nargiaf=\ang{\args^\fix,\args^?,\defrel}$ be an arg-IAF, 
    we construct the target $\rulprefix\nisaf=\unravelsaf$ as follows:

    \begin{itemize}
        \item $\lanset=\{p_{x},\lnot p_{x} \mid x \in \args^\fix\cup \args^?\}$.
        \item  $\overline{\cdot}$ is given by two clauses:
        \begin{itemize}
            \item[(C1)] $p \in \overline{\lnot p}$ and $\lnot p \in \overline{p}$ for every $p,\lnot p \in \lanset$.\footnote{This clause is only included to make sure that each formula has at least one contradictory, as required in Definition \ref{def:aspic-theory}.}
            \item[(C2)]  $p_{x} \in \overline{p_{y}}$ iff $\ang{x,y}\in \defrel$.
        \end{itemize}
         \item $\rulset_d^?=\{\langle \{\}, p_{x}\rangle \mid x \in  \args^?\}$.\footnote{Recall that there can be rules with empty antecedent/body (Remark \ref{remark:premise-less-rules}).}
            \item $\kb_p=\{p_{x} \mid x \in \args^\fix\}$.
        \item The other components of $\rulprefix\nisaf$ are empty.
    \end{itemize}
Observe that the set of all arguments that might appear in the completions of $\rulprefix\nisaf$ is $\kb_p\cup\rulset_d^?$. More complex arguments cannot be formed because all rules of the defined argumentation system have empty antecedents. The function $i:(\args^\fix \cup \args^?)\to (\kb_p\cup\rulset_d^?)$ that establishes the equivalence between the two sets of completions is just 
$$
    i(x)= \left\{\begin{array}{lr}
        p_x, & \text{if } x \in \args^\fix\\
        & \\
        \Rightarrow p_x, & \text{if } x \in \args^?\\
        \end{array}\right.
$$

It is immediate to check that $i$ is bijective, since $i[\args^\fix]=\kb_p$, and $i[\args^?]=\rulset_d^?$.

We then show that 
\begin{center}
    $\{\ang{i[\args^\ast],i[\defrel^\ast]}\mid \ang{\args^\ast,\defrel^\ast}\in \compfun(\argprefix\niaf)\}=\compfun(\rulprefix\nisaf)$.
\end{center}

Which amounts to showing two things:
\begin{itemize}
    \item[(T1)]\label{t1} For every $\comp \in \compfun(\argprefix\niaf)$, $\ang{i[\args^\ast],i[\defrel^\ast]}\in \compfun(\rulprefix\nisaf)$; and
    \item[(T2)] For every  $\ang{\args',\defrel'} \in \compfun(\rulprefix\nisaf)$, there is a $\comp \in \compfun(\argprefix\niaf)$ such that $\ang{i[\args^\ast],i[\defrel^\ast]}=\ang{\args',\defrel'}$.
\end{itemize}

For (T1), suppose that $\comp \in \compfun(\argprefix\niaf)$. This implies, by definition of completion for arg-IAFs, that
\begin{itemize}
    \item[(a)]\label{a} $\args^\fix \subseteq \args^\ast \subseteq \args^\fix \cup \args^?$; and
    \item[(b)]\label{b} $\defrel^\ast=\defrel_{\restrict \args^\ast}$.
\end{itemize}

We need to show two things. 

\begin{itemize}
    \item[(i)]\label{1} $\kb_p \subseteq i[\args^\ast] \subseteq \kb_p \cup \rulset^?_d$; and
     \item[(ii)]\label{2} $i[\defrel^\ast]=\defrel(\rulprefix\nisaf)_{\restrict i[\args^\ast]}$.
\end{itemize}

Note that (i) is implied by (a) and the fact that $i$ is monotonic, i.e., the fact that for all $X,Y \subseteq \args^\fix\cup\args^?$, we have that $X\subseteq Y$ implies $i[X]\subseteq i[Y]$. We leave details to the reader. \par

As for (ii), let us show both inclusions. From left to right, suppose $\ang{i(x),i(y)}\in i[\defrel^\ast]$, which implies $\ang{x,y}\in \defrel^\ast$ (by definition of $i$),  and then by (b) that $\ang{x,y}\in \defrel$ and $x,y \in \args^\ast$ (and hence $i(x),i(y) \in i[\args^\ast]$). If $\ang{x,y}\in \defrel$, then we have by (C2) that $p_x \in \overline{p_y}$. Reasoning by cases on the membership of $i(x)$ and $i(y)$ to $\kb_p$ and $\rulset_d^?$, and applying the ASPIC$^{+}$definition of defeat, we can arrive to $\ang{i(x),i(y)}\in \defrel(\rulprefix\nisaf)$. Therefore, $\ang{i(x),i(y)}\in \defrel(\rulprefix\nisaf)_{\restrict i[\args^\ast]}$. From right to left, suppose that $\ang{i(x),i(y)}\in \defrel(\rulprefix\nisaf)_{\restrict i[\args^\ast]}$, which implies that $x,y \in \args^\ast$ (by the definition of $i$). Since $\ang{i(x),i(y)}\in \defrel(\rulprefix\nisaf)$, we can arrive, reasoning by cases on the definition of defeat to $p_x \in \overline{p_y}$ (note that there are no undercuts because $\namefun$ is empty). The latter implies by (C2) that $\ang{x,y}\in \defrel$, which together with $x,y \in \args^\ast$ (that we already knew) and (b) implies $\ang{x,y}\in \defrel^\ast$, which implies $\ang{i(x),i(y)}\in i[\defrel^\ast]$ (by definition of $i$).
\par \medskip

As for (T2), let $\ang{\args',\defrel'} \in \compfun(\rulprefix\nisaf)$. We need to find a completion of $\argprefix\niaf$ whose $i$-image is $\ang{\args',\defrel'}$. Let us simply show that $\ang{i^{-1}[\args'],i^{-1}[\defrel']}$ is indeed a completion of $\argprefix\niaf$. This amounts to showing:

\begin{itemize}
    \item[(T3)] $\args^\fix \subseteq i^{-1}[\args']\subseteq \args^\fix \cup \args^?$; and
    \item[(T4)] $i^{-1}[\defrel']=\defrel_{\restrict i^{-1}[\args']}$.
\end{itemize}
For the first inclusion of (T3), suppose that $x \in \args^\fix$, which implies $p_x\in \kb_p$ (by definition of $\rulprefix\nisaf$), which implies $p_x \in \args'$ (by the definition of ASPIC$^{+}$ arguments and the definition of completions for rul-ISAFs),  which implies by definition of $i$ that $i^{-1}(p_x)=x\in i^{-1}[\args']$. For the second inclusion, suppose that $x \in i^{-1}[\args']$. We proceed by cases on the value of $i(x)$. If $i(x)=p_x$, then $x \in \args^\fix$ (by definition of $i$).  If $i(x)=\Rightarrow p_x$, then $x \in \args^?$ (by definition of $i$). Hence $x \in \args^\fix \cup \args^?$. \par 

For the left-to-right inclusion of (T4), suppose that $\ang{x,y}\in i^{-1}[\defrel']$. Reasoning by cases on the membership of $i(x)$ and $i(y)$ to $\kb_p \cup \rulset_d^?$ and applying the definition of $\rulprefix\nisaf$ we can arrive to $p_x \in \overline{p_y}$, which implies by (C2) that $\ang{x,y}\in \defrel$ (it is immediate to see that $x,y \in i^{-1}[\args']$). The right-to-left inclusion follows an analogous line of reasoning in the opposite direction.
\end{proof}

The naturally ensuing question is whether the inverse claim holds, i.e. if arg-IAFs are at least as expressive as rul-ISAFs. The answer is negative. 

\begin{theorem}\label{thm:rul-ISAFs-n-less-expressive-arg-IAFs}
 $\text{rul-ISAFs} \not \lexpressive \text{arg-IAFs}$.
\end{theorem}

\begin{proof}
 A simple rul-ISAF satisfying the statement is obtained by setting $\rulset^{F}_s=\rulset_s^{?}=\varnothing$, $\rulset^{F}_d=\{\ang{\{q\},r}\}$, and $\rulset_{d}^{?}=\{\ang{\{p\},q}\}$ (where $\lanset$ is the language of propositional logic and $\overline{\cdot}$ is given by classical negation), $\kb_{n}=\varnothing$, $\kb_{p}=\{p\}$, and $\preceq=\varnothing$. There are then two rule-completions and two associated AFs (two completions), namely, $\ang{\{p\},\varnothing}$ and $\ang{\{p, p\Rightarrow q, p \Rightarrow q \Rightarrow r \}, \varnothing}$. It is easy to show that no IAF has a set of completions equivalent to those AFs. Suppose there was such an IAF; then (the arguments corresponding to) $p\Rightarrow q$ and $p \Rightarrow q \Rightarrow r$ must belong to $\args^?$, because they don't appear in all completions, and therefore they cannot be certain. But then there should be a completion with $p\Rightarrow q$ and without $p \Rightarrow q \Rightarrow r$, but this is not the case (contradiction!).

\end{proof}

We can nonetheless achieve a positive expressivity result by extending arg-IAFs to their implicative extension of arg-IAFs presented in Section \ref{subsec:arg-IAFs}. The following theorem corrects Theorem 1 of \textcite{ai32023}, where something similar is claimed for a subclass of imp-arg-IAF that only admits singleton sets on the left-hand side of implicative dependencies.

\begin{theorem}\label{thm:rul-iafs-representation} 
$\text{rul-ISAFs }\lexpressive \text{ imp-arg-IAFs}$.
\end{theorem}

\begin{proof}
Let $\rulprefix\nisaf=\unravelsaf$ be a rule-incomplete structured AF with $\rulset=\rulset^{\fix}\cup \rulset^{?}$.

Now, we build the target imp-arg-IAF $\ang{\args^{\fix},\args^{?},\defrel,\Delta}$ as follows:

\begin{itemize}
\item $\args^{\fix}=\args(\nsaf^{\fix})$.

\item $\args^{?}= \args(\nsaf^{max})\setminus\args^{\fix}$.	

\item $\defrel=\defrel(\nsaf^{max})$.

\item $\Delta=\{\opimp(\Gamma,X)\mid \rulset^{?}(X)\subseteq \rulset^{?}(\Gamma)\}$.
\footnote{Recall that $\Gamma\subseteq \args^?$ (non-empty) and $X \in \args^?$ (Definition \ref{def:dep-arg-IAF}).}
\end{itemize}

Any injective function $i: \args(\nsaf^{max})\to \uni$ will do the job here for proving equivalence among the two sets of completions, so we just identify each $X \in \nsaf^{max}$ with its abstract name $i(X)$. Then, we need to show that both directions of the equality  $\compfun(\rulprefix\nisaf)=\compfun(\ang{\args^{\fix},\args^{?},\defrel,\Delta})$ hold:

\par \smallskip

\noindent [$\subseteq$] Suppose $\comp \in \compfun(\rulprefix\nisaf)$, which amounts by definition of $\compfun$ to

\begin{itemize}
\item[(H0)] $\comp=\daf(\nsaf^{\ast})$ for some $\nsaf^{\ast}\in \rulprefix\compfun(\nisaf)$.
\end{itemize}  Hence, we need to check that $\comp$ satisfies the conditions for being a completion of $\ang{\args^{\fix},\args^{?},\defrel,\Delta}$ (Section \ref{subsec:arg-IAFs}). To do so, we need to establish the following claims, whose proof we omit:

\begin{lemma}\label{lemmaThm1}
Let $\nsaf,\nsaf^{\prime}\in \rulprefix\compfun(\nisaf)$. Then:

\begin{itemize}
\item[(i)]\label{lemmaThm1.a} $\rulset(\nsaf)\subseteq \rulset(\nsaf^{\prime})$ implies $\args(\nsaf)\subseteq \args(\nsaf^{\prime})$.
\item[(ii)]\label{lemmaThm1.b} $\rulset(\nsaf)\subseteq \rulset(\nsaf^{\prime})$ implies $\defrel(\nsaf)= \defrel(\nsaf^{\prime})_{\restrict \args(\nsaf)}$.
\item[(iii)] For every $\Gamma\subseteq \args(\nsaf^{max})$ and every $X \in \args(\nsaf^{max})$ we have that:

\begin{center}
If $\Gamma\subseteq \args(\nsaf)$ and $\rulset^{?}(X)\subseteq \rulset^{?}(\Gamma)$, then $X \in \args(\nsaf)$.
\end{center}
\end{itemize}
\end{lemma}

Each of the conditions follows from (H0) and its corresponding item of the previous lemma.
In particular, the condition $\args^{\fix}\subseteq \args^{\ast}\subseteq \args^{\fix}\cup \args^{?}$ follows by verifying that item (i) of Lemma \ref{lemmaThm1} entails that $\args(\nsaf^{\fix})\subseteq \args^{\ast}\subseteq \args(\nsaf^{max})$. In the same way, item (ii) of Lemma \ref{lemmaThm1} entails $\defrel^{\ast}=\defrel(\nsaf^{max})_{\restrict  \args^{\ast}}$. Finally, item (iii) ensures that for all the dependencies listed in the definition, $\Gamma\subseteq \args^\ast$, then $X \in \args^\ast$.

\par \smallskip

[$\supseteq$]\footnote{This direction did not hold in a previous work by \textcite{ai32023}.} Suppose $\comp \in \compfun(\imparg)$, we have that:

\begin{itemize}

\item[(H1)] $\args^{\fix}\subseteq \args^{\ast}\subseteq \args^{\fix}\cup \args^{?}$.
\item[(H2)] $\defrel^{\ast}=\defrel_{\restrict \args^{\ast}}$.
\item[(H3)] For all $\opimp(\Gamma,X)\in \Delta$ if $\Gamma\subseteq \args^{\ast}$, then $X\in \args^{\ast}$.

\end{itemize}

Now, we need to find a rule-completion of $\nisaf$, $\nsaf^{\prime}$ such that\ $\daf(\nsaf^{\prime})=\comp$. We define $\nsaf^{\prime}=\unravelsafprime$ as the rule-completion determined by the set of rules:
\begin{center}
$\rulset^{\prime}=\rulset^{\fix}\cup \rulset^{?}(\args^{\ast})$.
\end{center}
Hence, $\lanset$ and $\kb$ are as in the original rul-ISAF $\rulprefix\nisaf$ and $\namefun^\ast$ and $\preceq^\ast$ are the original components relativised to the set of argument $\args(\lanset,\rulset',\kb)$ (see Definition \ref{def:rul-isaf-completions}). \par 
We need to check three things about this SAF, namely:

\begin{enumerate}
\item[(iv)] $\nsaf^{\prime}\in \rulprefix(\nisaf)$.
\item[(v)] $\args^{\ast}=\args(\nsaf^{\prime})$.
\item[(vi)] $\rel^{\ast}=\rel(\nsaf^{\prime}$).
\end{enumerate}

For (iv), we need to show $\rulset^{\fix}\subseteq \rulset^{\prime}\subseteq (\rulset^{\fix}\cup \rulset^{?})$. The first inclusion holds by definition of $\rulset^{\prime}$. For the second inclusion, suppose that $r \in \rulset^{\prime}$, which implies $r \in \rulset^{\fix}\cup \rulset^{?}(\args^{\ast})$ (by definition of $\rulset^{\prime}$), which implies $r \in \rulset^{\fix}\cup \rulset^{?}$ (by (H1) and definition of $\args^{?}$). \par \medskip

For (v), we show that both inclusions hold. [$\subseteq$] Suppose $X \in \args^{\ast}$, which implies by (H1) that $X \in \args^{\fix}\cup \args^{?}$. We proceed by cases. In the case that $X \in \args^{\fix}$, we have that $X \in \args(\nsaf^{\fix})$ (by definition of $\args^{\fix}$), which implies $X \in \args(\nsaf^{\prime})$ (by definition of $\rulset^{\prime}$ and Lemma \ref{lemmaThm1}.(i)). In the case that $X \in \args^{?}$, we have that $\rulset(X)\subseteq \rulset(\nsaf^{\prime})$ (by definition of $\rulset^{\prime}$), which implies $X \in \args(\nsaf^{\prime})$ (by definition of ASPIC$^{+}$-argument). [$\supseteq$]\footnote{This inclusion was unprovable in \cite{ai32023}.} Suppose $X \in \args(\nsaf^{\prime})$, this entails (by definition of $\nsaf^{\prime}$) that $\rulset(X)\subseteq \rulset^{\fix}\cup \rulset^{?}(\args^{\ast})$, which implies (by definition of $\rulset^{?}$) $\rulset^{?}(X)\subseteq \rulset^{?}(\args^{\ast})$. From this it follows (by definition of $\Delta$) that $\ang{\args^{\ast},X}\in \Delta$, which in turn entails (by definition of completion of imp-arg-IAFs) that $X \in \args^{\ast}$).

\par \medskip

For (vi),  we show that both inclusions hold. [$\subseteq$] Suppose $\ang{X,Y} \in \rel^{\ast}$, which implies (by definition completion for imp-arg-IAFs) that $\langle X,Y \rangle \in \args^{\ast}\times \args^{\ast}$ and $\ang{X,Y} \in \rel$. This entails (by (v) and the definition of $\rel$, respectively) that $\ang{X,Y}\in \args(\nsaf')\times \args(\nsaf')$ and $\ang{X,Y} \in \rel(\nsaf^{max})$. The latter entails, by Lemma \ref{lemmaThm1}.(ii) and by the fact that $\rulset(\nsaf')\subseteq\rulset(\nsaf^{max})$ that $\ang{X,Y}\in \rel(\nsaf')$. [$\supseteq$] Suppose that $\ang{X,Y}\in \rel(\nsaf')$. This implies (by (v) and Lemma \ref{lemmaThm1}.(ii), respectively) that  $X,Y \in \args^{\ast}$ and $\ang{X,Y}\in \nsaf^{max}$, which entails (by definition of $\rel$) that $X,Y \in \args^{\ast}$ and $\ang{X,Y}\in \rel$. It follows that $\ang{X,Y}\in \rel^{\ast}$ (by definition of completion for imp-arg-IAFs).
\end{proof}

Our next result establishes that imp-arg-IAFs are in fact strictly more expressive than rul-ISAFs. Its proof relies on the following lemma:

\begin{lemma}\label{lemma:caract-rul-completions} Let $\rulprefix\nisaf=\unravelsaf$ be a rul-ISAF, and let $\Gamma \subseteq \args(\nsaf^{max})$, $X \in \args(\nsaf^{max})$ we have that:

\begin{center}
    for every $\comp \in \compfun(\rulprefix\nisaf)$, $\Gamma\subseteq \args^\ast$ implies $X \in \args^\ast$ \\
    iff \\
    $\rulset^?(X)\subseteq \rulset^?(\Gamma)$.
\end{center}    
\end{lemma}

\begin{proof}
    Let $\rulprefix\nisaf=\unravelsaf$ be a rul-ISAF, and let $\Gamma \subseteq \args(\nsaf^{max})$, $X \in \args(\nsaf^{max})$. 
    \medskip 
    \par 

    [$\Rightarrow$] Suppose that 

     \medskip 
    \par 
    \centerline{ for every $\comp \in \compfun(\rulprefix\nisaf)$, $\Gamma\subseteq \args^\ast$ implies $X \in \args^\ast$ \quad (i)}

     \medskip 
    \par 
    Suppose also, reasoning by contradiction, that

     \medskip 
    \par 

    \centerline{$r \in \rulset^?(X)$, but $r \notin \rulset^?(\Gamma)$ \quad (ii) }

\medskip 
\par

    \noindent Now, let $\nsaf^\Gamma=\ang{\lanset, \contfun,\rulset^\fix\cup\rulset^?(\Gamma),\namefun^\ast,\kb,\preceq^\ast}$, which is a rule-completion of $\rulprefix\nisaf$. It is clear that $\Gamma \subseteq \args(\nsaf^\Gamma)$. However, since $r \notin \rulset(\nsaf^\Gamma$) (because of (ii)), we have that $X \notin \args(\nsaf^\Gamma)$, which contradicts (i). \par \medskip

    [$\Leftarrow$] Suppose that 
     \medskip 
    \par 

    \centerline{$\rulset^?(X)\subseteq \rulset^?(\Gamma)$ \quad (iii)}

     \medskip 
    \par 

    Let $\comp \in \compfun(\rulprefix\nisaf)$ (which means by definition of $\compfun$ that there is a rule-completion, name it $\nsaf^\ast$, such that $\naf(\nsaf^\ast)=\comp$).
Suppose that $\Gamma \subseteq \args^\ast$ (i.e., $\Gamma \subseteq \args(\nsaf^\ast$)). This implies, by definition of ASPIC-argument, that $\rulset^?(\Gamma)\subseteq \rulset(\nsaf^\ast)$, which implies by (iii), that $\rulset^?(X)\subseteq \rulset(\nsaf^\ast)$, which again by definition of argument implies $X \in  \args(\nsaf^\ast)$ (i.e., $X \in \args^\ast$).
\end{proof}

This allows us to prove the following:
\begin{theorem}\label{prop:impargiafs-not-lessexpresive-rulisafs}
   imp-arg-IAFs $\not \lexpressive$ rul-ISAFs
\end{theorem}

\begin{proof}[Proof (sketch)]

Let $\mathsf{imp}\text{-}\argprefix\niaf=\imparg$ where $\args^F=\varnothing$, $\args^?=\{a,b,c\}$, $\defrel=\varnothing$, and $\Delta=\{\ang{\{a,b\},c}\}$, this imp-arg-IAF can be shown not to have any rul-ISAF with an equivalent set of completions. First of all, note that such rul-ISAF should have no premises; otherwise $\args^F$ would not be empty. Now, reasoning towards contradiction, suppose that there is such a rul-ISAF, call it $\nisaf$. Then $i(c)\in \args(\nisaf)$ can be shown, using Lemma \ref{lemma:caract-rul-completions}, to have at least two premiseless rules, which are themselves members of all the completions where $i(c)$ appears. But then $\langle \{i(c)\},\varnothing\rangle$ is not a completion of $\nisaf$ (while it should by assumption!).
\end{proof}

%%%%%%==================
\subsection{Uncertain Premises}\label{subsec:uncertainpremises}
%%%%%%==================
\newcommand{\premcomp}{\ang{\lanset, \negfun, \rulset, \namefun, \kb^\ast, \preceq^{\ast}}}

A second option for grounding argument-incompleteness in the components of an ASPIC$^+$ structured argumentation framework is to consider sets of premises, i.e. knowledge bases, as possible sources of uncertainty, as suggested by \textcite{odekerken2023argumentative,odekerken2025argumentative}. We generalise their definition (which was restricted to a fragment of ASPIC$^+$) and adapt the terminology to ours for the sake of homogeneity.
\par 

\begin{definition}[After \citeauthor{odekerken2023argumentative}, \citeyear{odekerken2023argumentative}]\label{def:prem-isaf}
    A \textbf{premise-incomplete structured argumentation framework} (prem-ISAF) is a tuple $\premprefix\nisaf=\unravelsaf$ where every component is exactly as in a SAF except for the knowledge base, that comes split into four disjoint subsets $\kb=\kb_n^\fix\cup \kb_n^?\cup \kb_p^\fix\cup \kb_p^?$ representing, respectively, certain axioms, uncertain axioms, certain ordinary premises and uncertain ordinary premises. As before, we define $\kb^\fix=\kb_n^\fix\cup \kb_p^\fix$ and similarly for $\kb^?$. 
\end{definition}

Completions are then defined in the same way as they were for rule-incompleteness.

\begin{definition}[After \citeauthor{odekerken2023argumentative}, \citeyear{odekerken2023argumentative}\footnote{Premise-completions are called ``future argumentation frameworks'' by \textcite{odekerken2023argumentative}.}.]\label{def:prem-isaf-completion}
    
A \textbf{premise-completion} of $\premprefix\nisaf=\unravelsaf$ is any SAF $\nsaf^\ast=\premcomp$ where:
\begin{itemize}
    \item $\kb^\ast=\kb_n^\ast\cup \kb_p^\ast$ with:
\begin{itemize}
    \item $\kb_n^\fix \subseteq \kb_n^\ast \subseteq (\kb_n^\fix \cup \kb_n^?)$; and
    \item $\kb_p^\fix \subseteq \kb_p^\ast \subseteq (\kb_p^\fix \cup \kb_p^?)$.
\end{itemize}
\item  $\preceq^\ast=\preceq_{\restrict \args^\ast}$ with $\args^\ast=\args(\ang{\lanset,\rulset,\kb^\ast})$.

\end{itemize}
We denote as $\premprefix\mathsf{completions}(\nisaf)$ the set of all premise-completions of $\premprefix\nisaf$. Here again, two premise-completions, the minimal and the maximal ones, will be key for our proofs:
\begin{itemize}
\item $\nsaf^{\fix}$ is the premise-completion whose knowledge base is $\kb^{\fix}=\kb^{\fix}_n\cup \kb^{\fix}_p$.
\item $\nsaf^{max}$ is just $\premprefix\nisaf$ (i.e., the completion generated using certain and uncertain premises, i.e., all formulas in $\kb$).
\end{itemize}
Let $\nsaf \in \premprefix\compfun(\nisaf)$ and let $X \in \args(\nsaf)$, we define the set of uncertain (resp.\ certain) premises of argument $X$ as $\prem^{?}(X)=\prem(X)\cap \kb^{?}$ (resp.\ $\prem^{\fix}(X)=\prem(X)\cap \kb^{\fix}$). This function is lifted to sets of arguments by setting $\prem^?(\Gamma)=\bigcup_{Y \in \Gamma}\prem^?(Y)$ (and the same for $\fix$). 

\end{definition}

\par 
As before, we define the set of completions of $\premprefix\nisaf$ as the abstract AFs associated with its premise-completions.

\begin{definition}\label{def:prem-isaf-abstract-completions}
    The set of \textbf{completions} of a given $\premprefix\nisaf$ is defined as:

\begin{center}
    $\compfun(\premprefix\nisaf)=\{\daf(\nsaf^{\ast})\mid \nsaf^{\ast} \in \premprefix\mathsf{completions}(\nisaf)\}\text{.} $
\end{center}

\end{definition}

\begin{remark}
    The definition of prem-ISAFs and premise-completion assumes that it is known whether $\varphi$ belongs to $\kb_n$ or $\kb_p$ (i.e., whether $\varphi$ is an axiom or an ordinary premise) whenever $\varphi$ is indeed a premise. Hence, it is assumed that uncertainty only concerns the presence of premises, not their status.
\end{remark}

\begin{example}\label{example:prem-isaf} Let $\premprefix\nisaf_0=\unravelsaf$ where every component is defined just as in the SAF of Example \ref{example:saf} except for $\kb$, which is as follows:

\begin{itemize}
    \item $\kb^\fix_n=\{p,u\}$;
    \item $\kb^?_n=\varnothing$;
    \item $\kb^\fix_p=\{s\}$; and
     \item $\kb^?_p=\{w\}$.
\end{itemize}

This prem-ISAF has two premise-completions (i.e., one with premise $w$ and one without it) and two associated AFs:
\par \smallskip
\scalebox{1}{
\begin{tabular}{c|c}
  
\begin{tikzpicture}[->,>=stealth,shorten >=1pt,auto,node distance=1.4cm,
                thick,main node/.style={circle,draw,font=\bfseries},uncertain/.style={rectangle,draw,dashed,font=\bfseries}]

\node[sarg] (s) {$s$};
\node[sarg] (snr) [right=0.5cm of s] {$s\Rightarrow \lnot r$};

\node[sarg] (pq) [right=2cm of snr] {$p \Rightarrow q$};
\node[sarg] (p) [below=0.1cm  of pq] {$p$};

\node[sarg] (uns) [above of=s] {$u\sto \lnot s$};
\node[sarg] (u) [above=0.1cm of uns] {$u$};

\node[sarg] (wr) [below of=s] {$w\Rightarrow r$};
\node[sarg] (w) [below=0.1cm  of wr] {$w$};

\path[->] (snr) edge (pq);
\path[->] (uns) edge (snr);
\path[->] (uns) edge (s);
\path[->] (wr) edge (snr);
\end{tikzpicture}

     &

\begin{tikzpicture}[->,>=stealth,shorten >=1pt,auto,node distance=1.4cm,
                thick,main node/.style={circle,draw,font=\bfseries},uncertain/.style={rectangle,draw,dashed,font=\bfseries}]

\node[sarg] (s) {$s$};
\node[sarg] (snr) [right=0.5cm of s] {$s\Rightarrow \lnot r$};

\node[sarg] (pq) [right=2cm of snr] {$p \Rightarrow q$};
\node[sarg] (p) [below=0.1cm  of pq] {$p$};

\node[sarg] (uns) [above of=s] {$u\sto \lnot s$};
\node[sarg] (u) [above=0.1cm of uns] {$u$};

\node (wr) [below of=s] {};
\node (w) [below=0.1cm  of wr] {};

\path[->] (snr) edge (pq);
\path[->] (uns) edge (snr);
\path[->] (uns) edge (s);
\end{tikzpicture}

\end{tabular}
}
    
\end{example}

As with uncertain rules, we can show that uncertain knowledge bases only generate argument-incompleteness at the abstract level.

\begin{proposition}\label{prop:prem-ISAFs-arg-incompleteness}
   Let $\premprefix\nisaf$ be a prem-ISAF, and let $\ang{\args_1,\defrel_1},\ang{\args_2,\defrel_2}\in \compfun(\premprefix\nisaf)$. Then, for every $X,Y \in \args_1\cap \args_2$, $\ang{X,Y} \in \defrel_1$ iff $\ang{X,Y}\in \defrel_2$.
\end{proposition}

 \begin{proof}
      The proof is very similar to that of Proposition \ref{prop:rul-ISAFs-arg-incompleteness}. The key point is that the presence/absence of a set of premises does not affect the presence/absence of defeats, given that the involved arguments are present.
 \end{proof}
Also, analogously to rule-incompleteness, knowledge base-incompleteness is expressive enough to subsume standard arg-IAFs at the abstract level.

\begin{theorem}\label{thm:arg-IAFs-lexpressive-prem-ISAFs}
        $\text{arg-IAFs } \lexpressive \text{ prem-ISAFs}$.
\end{theorem}

 \begin{proof}
 Let $\nargiaf=\ang{\args^\fix,\args^?,\defrel}$ be an arg-IAF, 
     we construct the target $\premprefix\nisaf=\unravelsaf$ as follows:

     \begin{itemize}
         \item $\lanset=\{p_{x},\lnot p_{x} \mid x \in \args^\fix\cup \args^?\}$.
        \item  $\overline{\cdot}$ is given by two clauses:
         \begin{itemize}
             \item[(C1)] $p \in \overline{\lnot p}$ and $\lnot p \in \overline{p}$ for every $p,\lnot p \in \lanset$.
             \item[(C2)]  $p_{x} \in \overline{p_{y}}$ iff $\ang{x,y}\in \defrel$.
           
         \end{itemize}
         
             \item $\kb^\fix_p=\{p_{x} \mid x \in \args^\fix\}$.
              \item $\kb^?_p=\{p_{x} \mid x \in \args^?\}$.
         \item The other components of $\premprefix\nisaf$ are empty.
     \end{itemize}

 Note that all arguments of $\premprefix\nisaf$ (i.e., arguments of $\args(\nsaf^{max}$)) happen to be formulas since there are no inference rules. The function $i:\args(\nsaf^{max}) \to (\args^\fix \cup \args^?)$ is defined as $i(p_x)=x$. It is easy to check that $i$ is bijective and induces a bijection among the sets of arguments of completions of both formalisms.  As for the defeat relation, it is enough to use (C2) to check that $\ang{x,y}\in \defrel(\nsaf^{max})$ iff $\ang{i(x),i(y)}\in \defrel$. Hence, $i$ is a function satisfying the conditions of Definition \ref{def:expressivity}.
 \end{proof}

The naturally ensuing question is whether arg-IAFs are at least as expressive as prem-ISAFs. Again, the answer is negative. 

\begin{theorem}\label{prop:prem-isafs-negative}
    $\text{prem-ISAFs } \not \lexpressive \text{ arg-IAFs.}$
\end{theorem}

\begin{proof} Consider a prem-ISAF $\unravelsaf$ where $\lanset=\{p,q,r\}$, $\kb_p^\fix=\{p\}$, $\kb_p^?=\{q\}$, $\rulset_d=\{\tuple{q,r}\}$, $\overline{\cdot}$ is given by classical negation, and the rest of the components are empty. We then get two completions, namely $\tuple{\{p\},\varnothing}$ and $\tuple{\{p,q, q\Rightarrow r\},\varnothing}$, which can be easily shown not to correspond to the set of completions of any $\argprefix\niaf$. 
\end{proof}

As for the case of rule-incompleteness, it is, however, possible to retrieve a sufficient level of expressivity by adding implicative dependencies to the abstract framework. This follows, by the transitivity of $\lexpressive$ (Fact \ref{fact:lespressive-is-a-preorder}), from Theorem \ref{thm:rul-iafs-representation} and from Theorem \ref{thm:rule-more-prem}, proved below, establishing that rul-ISAFs are more expressive than prem-ISAFs. However, it is possible to prove this result directly in the same way as for Theorem \ref{thm:rul-iafs-representation}. We provide the direct proof, since it is informative about the technical aspects of lifting ASPIC$^{+}$ systems to AFs.

\begin{theorem}\label{thm:prem-iafs-representation}
   $\text{prem-ISAFs } \lexpressive \text{ imp-arg-IAFs.}$
\end{theorem}

 \begin{proof}
 Let $\premprefix\nisaf=\unravelsaf$ be a premise-incomplete structured AF with $\kb=\kb^{\fix}_n \cup \kb^{\fix}_p \cup \kb^{\fix}_n \cup \kb^{\fix}_p$. We define the target imp-arg-IAF $\ang{\args^{\fix},\args^{?},\defrel,\Delta}$ as follows:

 \begin{itemize}
 \item $\args^{\fix}=\args(\nsaf^{\fix})$.

 \item $\args^{?}= \args(\nsaf^{max})\setminus\args^{\fix}$.

 \item $\defrel=\defrel(\nsaf^{max})$.

 \item $\Delta=\{\opimp(\Gamma,X) \mid \prem^{?}(X)\subseteq \prem^{?}(\Gamma)\}$.\footnote{Recall that $\Gamma\subseteq \args^?$ (non-empty) and $X \in \args^?$ (Definition \ref{def:dep-arg-IAF}).}
 \end{itemize}

 Any injective function $i: \args(\nsaf^{max})\to \uni$ will do the job here for proving equivalence among the two sets of completions, so we just identify each $X \in \nsaf^{max}$ with its abstract name $i(X)$.  We show that both directions of the equality  $\compfun(\premprefix\nisaf)=\compfun(\ang{\args^{\fix},\args^{?},\defrel,\Delta})$ hold:

 \par \smallskip

 \noindent [$\subseteq$] Suppose $\comp \in \compfun(\premprefix\nisaf)$, which amounts by definition of $\compfun$ to

 \begin{itemize}
 \item[(H0)] $\comp=\daf(\nsaf^{\ast})$ for some $\nsaf^{\ast}\in \premprefix\compfun(\nisaf)$
 \end{itemize}  Hence, we just need to check that $\comp$ satisfies the conditions for being a completion of $\ang{\args^{\fix},\args^{?},\defrel,\Delta}$. To do so, we need to establish the following claims, whose proof we omit:

 \begin{lemma}\label{lemmaThm2}
 Let $\nsaf,\nsaf^{\prime}\in \premprefix\compfun(\nisaf)$. Then:

 \begin{enumerate}
 \item[(i)] $\kb(\nsaf)\subseteq \kb(\nsaf^{\prime})$ implies $\args(\nsaf)\subseteq \args(\nsaf^{\prime})$.
 \item[(ii)] $\kb(\nsaf)\subseteq \kb(\nsaf^{\prime})$ implies $\defrel(\nsaf)= \defrel(\nsaf^{\prime})_{\restrict \args(\nsaf)}$.
 \item For every $\Gamma\subseteq \args(\nsaf^{max})$ and every $X \in \args(\nsaf^{max})$ we have that:

 \begin{center}
 If $\Gamma\subseteq \args(\nsaf)$ and $\prem^{?}(X)\subseteq \prem^{?}(\Gamma)$, then $X \in \args(\nsaf)$.
 \end{center}
 \end{enumerate}
 \end{lemma}

 As for Lemma \ref{lemmaThm1} of Theorem \ref{thm:rul-iafs-representation}, each item of the previous lemma guarantees the satisfaction of the corresponding condition for being a completion. 

 \par \smallskip

 \newcommand{\unravelpremprime}{\ang{\lanset, \negfun, \rulset, \namefun, \kb^{\prime}, \preceq^{\prime}}}

 [$\supseteq$] Suppose $\comp \in \compfun(\imparg)$, we have that:

 \begin{itemize}

 \item[(H1)] $\args^{\fix}\subseteq \args^{\ast}\subseteq \args^{\fix}\cup \args^{?}$.
 \item[(H2)] $\defrel^{\ast}=\defrel_{\restrict \args^{\ast}}$.
 \item[(H3)] For all $\opimp(\Gamma,X)\in \Delta$ if $\Gamma\subseteq \args^{\ast}$, then $X\in \args^{\ast}$.

 \end{itemize}

 Now, we need to find a premise-completion of $\premprefix\nisaf$, $\nsaf^{\prime}$, such that\ $\daf(\nsaf^{\prime})=\comp$. We define $\nsaf^{\prime}=\unravelpremprime$ where the only component that has to be specified is $\kb^{\prime}=\kb_n^{\prime}\cup\kb_p^{\prime}$ where:

 \begin{center}
 $\kb^{\prime}_n=\kb^{\fix}_n\cup (\prem^{?}(\args^{\ast})\cap \kb_n^?)$; and \\
 $\kb^{\prime}_p=\kb^{\fix}_p\cup (\prem^{?}(\args^{\ast})\cap \kb_p^?)$.
 \end{center}

 We need to check three things about $\nsaf^\prime$, namely:

 \begin{enumerate}
 \item[(iv)] $\nsaf^{\prime}\in \premprefix\compfun(\nisaf)$.
 \item[(v)] $\args^{\ast}=\args(\nsaf^{\prime})$.
 \item[(vi)] $\rel^{\ast}=\rel(\nsaf^{\prime}$).
 \end{enumerate}

 For (iv), we need to show \begin{itemize}
     \item $\kb_n^\fix \subseteq \kb_n^\prime \subseteq (\kb_n^\fix \cup \kb_n^?)$.
     \item $\kb_p^\fix \subseteq \kb_p^\prime \subseteq (\kb_p^\fix \cup \kb_p^?)$.
 \end{itemize} The left-hand side inclusions ($\kb_n^\fix \subseteq \kb_n^\prime$ and $\kb_p^\fix \subseteq \kb_p^\prime$) hold by definition of $\kb^{\prime}$. As for the right-hand side inclusion, suppose that $\varphi \in \kb_n^{\prime}$ (resp.\ $\varphi \in \kb_p^{\prime}$), which implies, by definition of $\kb^{\prime}$, $\varphi \in \kb_n^{\fix}\cup (\prem^{?}(\args^{\ast})\cap \kb_n)$ (resp.\ $\varphi \in \kb_p^{\fix}\cup (\prem^{?}(\args^{\ast})\cap \kb_p)$). This entails, by (H1) and definition of $\args^{?}$, that $\varphi \in \kb_n^{\fix}\cup \kb_n^{?}$ (resp.\ $\varphi \in \kb_p^{\fix}\cup \kb_p^{?}$). \par \medskip

 For (v), we show that both inclusions hold. [$\subseteq$] Suppose $X \in \args^{\ast}$, which implies by (H1) that $X \in \args^{\fix}\cup \args^{?}$. We proceed by cases. In the case that $X \in \args^{\fix}$, we have that $X \in \args(\nsaf^{\fix})$ (by definition of $\args^{\fix}$), which implies $X \in \args(\nsaf^{\prime})$ (by definition of $\kb^{\prime}$ and Lemma \ref{lemmaThm2}.(i)). In the case that $X \in \args^{?}$, we have that $\prem(X)\subseteq \kb(\nsaf^{\prime})$ (by definition of $\kb^{\prime}$), which implies $X \in \args(\nsaf^{\prime})$ (by definition of ASPIC$^{+}$-argument). [$\supseteq$]
 Suppose $X \in \args(\nsaf^{\prime})$, which implies (by definition of $\nsaf^{\prime}$) that $\prem(X)\subseteq \kb^{\fix}\cup \prem^{?}(\args^{\ast})$. This entails (by definition of $\prem^{?}$) $\prem^{?}(X)\subseteq \prem^{?}(\args^{\ast})$, and in turn (by definition of $\Delta$) that $\ang{\args^{\ast},X}\in \Delta$. From here it follows (by definition of completion of imp-arg-IAFs) that $X \in \args^{\ast}$.

 \par \medskip

 For 3,  we show that both inclusions hold. [$\subseteq$] Suppose $\ang{X,Y} \in \rel^{\ast}$ which implies (by definition completion for imp-arg-IAFs) that $\langle X,Y \rangle \in \args^{\ast}\times \args^{\ast}$ and $\ang{X,Y} \in \rel$ which implies (by (v) and definition of $\rel$, respectively) that $\ang{X,Y}\in \args(\nsaf')\times \args(\nsaf')$ and $\ang{X,Y} \in \rel(\nsaf^{max})$ (which implies, by Lemma \ref{lemmaThm2}.(ii) and by the fact that $\kb(\nsaf')\subseteq\kb(\nsaf^{max})$) that $\ang{X,Y}\in \rel(\nsaf')$. [$\supseteq$] Suppose that $\ang{X,Y}\in \rel(\nsaf')$ which implies (by (v) and Lemma \ref{lemmaThm2}.(ii), respectively) that $X,Y \in \args^{\ast}$ and $\ang{X,Y}\in \nsaf^{max}$, which implies (by definition of $\rel$) that  $X,Y \in \args^{\ast}$ and $\ang{X,Y}\in \rel$, which implies $\ang{X,Y}\in \rel^{\ast}$ (by definition completion for imp-arg-IAFs).
 \end{proof}

The next question is whether prem-ISAFs are sufficient to express imp-arg-IAFs. The answer is negative, and, as for the previous result, it is possible to prove this indirectly. The direct proof we provide below (Theorem \ref{thm:imp-strictly-more-prem}) requires the following Lemma.

\begin{lemma}\label{lemma:caract-prem-completions} Let $\premprefix\nisaf=\unravelsaf$ be a prem-ISAF, and let $\Gamma \subseteq \args(\nsaf^{max})$, $X \in \args(\nsaf^{max})$ we have that:

\begin{center}
    for every $\comp \in \compfun(\premprefix\nisaf)$, $\Gamma\subseteq \args^\ast$ implies $X \in \args^\ast$ \\
    iff \\
    $\prem^?(X)\subseteq \prem^?(\Gamma)$.
\end{center}    
\end{lemma}

\begin{proof}
    From left to right, suppose that \par \smallskip \noindent
  (H1) for every $\comp \in \compfun(\premprefix\nisaf)$, $\Gamma\subseteq \args^\ast$ implies $X \in \args^\ast$ \par \smallskip \noindent
    and, reasoning by contradiction, suppose that \par \smallskip \noindent
    (H2) there is a $\varphi\in \prem^?(X)$ such that $\varphi \notin \prem^?(\Gamma)$ \par \smallskip

    Note that $\nsaf_0=\ang{\lanset,\overline{\cdot},\rulset,\namefun,\kb^\fix \cup \prem?(\Gamma),\preceq^\ast}$ is a prem-completion of $\premprefix\nisaf$,  where $\preceq^\ast$ is the restriction of $\preceq$ to $\args(\ang{\lanset,\rulset,\kb^\fix \cup \prem?(\Gamma)})$ as established in Definition \ref{def:prem-isaf-completion}. Moreover, it is clear that $\Gamma\subseteq \args(\nsaf_0)$ (by definition of argument and $\nsaf_0$). However, $X \notin \args(\nsaf_0)$ (because of (H3) and the definition of $\nsaf_0$). But since $\naf(\nsaf_0)\in \compfun(\premprefix\nisaf)$, we have that (H1) must be false (contradiction!). \par \medskip

    From right to left, suppose that 

\par \smallskip \noindent
    (H3) $\prem^?(X)\subseteq \prem^?(\Delta)$, and \par \smallskip \noindent
 \par \smallskip \noindent
    (H4) $\ang{\args_1,\defrel_1}\in \compfun(\premprefix\nisaf)$ and $\Gamma\subseteq \args_1$ (with the aim to show that $ X \in \args_1$). \par \smallskip

    From (H4), we have that $\ang{\args_1,\defrel_1}\in \naf(\nsaf_1)$ for some $\nsaf_1\in \premprefix\compfun(\nisaf)$ (by Definition \ref{def:prem-isaf-abstract-completions}). Also from (H4) and by the definition of ASPIC$^+$ argument, we can deduce that $\prem^?(\Gamma)\subseteq \kb(\nsaf_1)$. By (H3), we get $\prem^?(X)\subseteq \kb(\nsaf_1)$, which implies $X \in \args_1$.
\end{proof}

This allows proving the following result.

\begin{theorem}\label{thm:imp-strictly-more-prem}
    $\text{imp-arg-IAFs }\not \lexpressive \text{ prem-ISAFs.}$
\end{theorem}

\begin{proof}
    A simple witness of this existential claim is $\ang{\args^\fix,\args^?,\defrel, \Delta}$ where:
    \begin{itemize}
        \item $\argsf=\varnothing$.
        \item $\args^?=\{a,b,c\}$.
       \item $\defrel=\varnothing$.
        \item $\Delta=\{\opimp(\{a\},b), \opimp(\{c\},b)\}$.
    \end{itemize}

We have five completions (all with an empty defeat relation), namely $\{a,b\}$, $\{b\}$, $\{b,c\}$, $\{a,b,c\}$ and $\varnothing$. Now, reasoning toward contradiction, let us suppose that there is a prem-ISAF $\premprefix\nisaf=\unravelsaf$ with a set of completions that is equivalent to the previous one. For simplicity, we identify each $X \in \nsaf^{max}$ with $i(X)\in \{a,b,c\}$ in the remainder of the proof. Two observations about $\premprefix\nisaf$ can be derived by noting that there are no certain arguments. First, there are no rules with an empty antecedent and hence no argument with an empty set of premises, since this argument would be certain. Second, $\kb^\fix=\varnothing$, as otherwise we would again have at least a certain argument. Putting both observations together, we have that each argument has at least one uncertain premise. By Lemma \ref{lemma:caract-prem-completions}, we have that the previous set of completions implies that: $\prem^?(b)\subseteq \prem^?(a)$, $\prem^?(b)\subseteq \prem^?(c)$ but $\prem^?(a)\nsubseteq \prem^?(b)\cup \prem^?(c)$. If there are only uncertain premises and no argument without premises, then there are $\varphi_1,\varphi_2 \in \prem(a)$ such that $\varphi_1 \notin \prem(b)\cup \prem(c)$. It is easy to check that $\varphi_1 \neq a$, $\varphi_1 \neq b$ and $\varphi_1 \neq c$. However, $\varphi_1$ must appear in a completion of $\premprefix\nisaf$, hence the sets of completions of $\argprefix\niaf$ and $\premprefix\nisaf$ are not equivalent (contradiction!).
\end{proof}

\subsection{Comparing Uncertain Rules and Uncertain Premises}

One important question is how rule-incompleteness compares with premise-incompleteness, i.e whether rul-ISAFs are more or less expressive than prem-ISAFs or else incomparable with them. This section shows that rul-ISAFs are strictly more expressive. The first part of this proof consists of the following result.

\begin{theorem}\label{prop:prem-no-more-rule}     
    $\text{rul-ISAFs }\not \lexpressive \text{ prem-ISAFs.}$
\end{theorem}

\begin{proof}
        The target rul-ISAF is defined as follows:

        \begin{itemize}
            \item $\lanset=\{p_x, \lnot p_x \mid x \in \{a,b,c\}\}$.
            \item $\overline{\cdot}$ is given by $\lnot$.
            \item $\rulset_d^?=\{\ang{\{\},p_b},\ang{ \{p_b\} , p_a}, \langle \{p_b\} ,p_c\rangle\}$.
            \item  The other components are empty.
        \end{itemize}
        
The completions of this rul-ISAF are equivalent to the ones of the imp-arg-IAF defined in the proof of Theorem \ref{thm:imp-strictly-more-prem}, hence they are not equivalent to the completions of any prem-ISAF.
    \end{proof}

We now move to the most difficult result of this section, showing that $\text{prem-ISAFs } \lexpressive \text{ rul-ISAFs}$. To prove this, we need to take a detour. First, we show that prem-ISAFs are as expressive as one of their proper subclasses, which we will call \textit{tidy} prem-ISAFs, defined as follows.

\begin{definition}\label{def:tidypremisafs}
Let us say a prem-ISAF $\premprefix\nisaf=\unravelsaf$ is \textbf{tidy} when $ \kb \cap\{\psi \mid \ang{\{\},\psi}\in \rulset\}=\varnothing$, i.e., when there is no $\varphi \in \lanset$ such that $\varphi$ is a premise and the head of a premiseless rule simultaneously. We denote by t-prem-ISAFs the class of all tidy prem-ISAFs.
\end{definition}

Equal expressivity is established by the next result.

\begin{lemma}\label{lemma:tidyasexpressive}
    $\text{prem-ISAFs } \lexpressive \text{ t-prem-ISAFs.}$
\end{lemma}
\newcommand{\pisaf}{\premprefix\nisaf}
\newcommand{\pisafp}{\premprefix\nisaf}

\begin{proof}
\newcommand{\repfun}{\mathsf{REP}}
\newcommand{\ksaf}{\ang{\lanset,\contfun,\rulset,\namefun,\kb^\ast,\preceq^{\ast}}}
\newcommand{\ksafp}{\ang{\lanset',\contfun',\rulset',\namefun',\kb^\ast,\preceq^{\ast \prime}}}

Let $\premprefix\nisaf=\unravelsaf$ with $\kb=\kb_n^F\cup \kb^?_n\cup \kb_p^F\cup\kb_p^?$, its set of repetitions is defined as $\mathsf{REP}=\kb\cap\{\varphi \in \lanset\mid\ang{\{\},\varphi}\in \rulset\}$, that is, the set of all formulas that work both as premises and as heads of premiseless rules. Hence $\premprefix\nisaf$ is tidy iff $\mathsf{REP}=\varnothing$. We will use a fresh syntactic copy of $\lanset$ to get rid of repetitions, that is, a set $c(\lanset)=\{\varphi'\mid \varphi \in \lanset\}$ with $\lanset\cap c (\lanset)=\varnothing$. We further assume that $(\cdot)':\lanset \to c(\lanset)$ is injective. This map can be obtained, in the case of propositional languages, by using fresh copies of variables and Boolean constants, and providing a straightforward recursive translation. Hence, for instance, we would have that $\big(p \to (q \leftrightarrow (p\land \lnot s))\big)'=p'\to (q' \leftrightarrow (p'\land \lnot s')) $. The details of defining $(\cdot)'$ for each language are left unspecified. \\
We further define the following translation function $\tau$:
\begin{itemize}
     \item For formulas:         $
    \tau_F(\psi)= \left\{\begin{array}{lr}
   \psi'  \text{\quad if } \psi \in \repfun;
        & \\
        \psi \text{ \quad otherwise.}\\
        \end{array}\right.
$
            \item For premiseless rules: $\tau_R(\ang{\{\},\psi})=\ang{\{\},\tau_F(\psi)}$.

        \item For the rest of the rules: $\tau_R(\ang{\{\varphi_1,...,\varphi_n\}, \varphi}=\ang{\{\tau_F(\varphi_1),...,\tau_F(\varphi_n)\}, \varphi}$. 

        \item For atomic arguments: $\tau_A(\varphi)=\varphi$.
        
       \item For simple premiseless arguments: $\tau_A(\Rrightarrow \psi)=\Rrightarrow \tau_F(\psi)$.
       
\item For the rest of arguments: $\tau_A(A_1,...,A_n \Rrightarrow \varphi)=\tau_A(A_1),...,\tau_A(A_n)\Rrightarrow \varphi$.
    \end{itemize}

    We drop the subscript from $\tau$ whenever the context is clear. \par \smallskip
Using $\tau$, the equivalent tidy prem-ISAF we are after can be defined as $\premprefix\nisaf'=\ang{\lanset',\contfun',\rulset',\namefun,'\kb',\preceq'}$ where:

\begin{itemize}
    \item $\lanset'=\lanset\cup c(\lanset)$.

    \item $\contfun'=\contfun\cup \{\ang{\varphi,\psi'},\ang{\varphi',\psi},\ang{\varphi',\psi'}\mid \varphi \in \overline{\psi}\}$.

    \item $\rulset'=(\rulset \setminus \{\ang{\{\},\varphi} \mid \varphi \in \repfun\})\cup \{\tau(r)\mid r \in \rulset\}$.

    \item $\namefun'=\namefun\cup\{\ang{\tau(r),\varphi}\mid \ang{r,\varphi}\in \namefun\}$.

    \item $\kb'=\kb$.

\item $\preceq'=(\preceq\cup\{\ang{A,\tau(B)},\ang{\tau(A),B},\ang{\tau(A),\tau(B)}\mid A\preceq B\})_{\restrict\args(\lanset',\rulset',\kb')}$.

\end{itemize}

We now show that $\tau_A$ works as an equivalence function between the completions of both prem-ISAFs (see Definition \ref{def:expressivity}). We split this into several claims:
\par \medskip
    
\noindent [\textbf{Claim 1}. For every $X\in \args(\premprefix\nisaf), \tau_A(X)\in \args(\premprefix\nisaf')$]\footnote{It is easy to check that $\bigcup_{\ang{\args^\ast,\defrel^\ast}\in \compfun(\premprefix\nisaf)}\args=\args(\premprefix\nisaf)$, and $\bigcup_{\ang{\args^\ast,\defrel^\ast}\in \compfun(\premprefix\nisaf')}\args=\args(\premprefix\nisaf')$.}
\par \smallskip

Let $X\in \args(\premprefix\nisaf)$, we continue by induction on $X$. \\

\noindent[Case: $X=\varphi\in \kb$] We have that $\tau_A(X)=\varphi$ by definition of $\tau_A$. It then holds that $\tau_A(\varphi) \in \kb$ by hypothesis and, since $\kb=\kb'$ by definition of $\premprefix\nisaf'$, we obtain $\tau_A(X)\in \args(\premprefix\nisaf')$ by definition of argument.  
\par \smallskip
\noindent[Case: $X=\Rrightarrow \varphi$] We distinguish two subcases. First, if $\varphi\notin\repfun$, $\tau(X)=\Rrightarrow\varphi$, but since $\ang{\{\},\varphi} \in \rulset'$ by definition of $\rulset'$, we have that $\tau_A(X)\in \args(\premprefix\nisaf')$. Second, if $\varphi\in\repfun$, $\tau(X)=\,\Rrightarrow\varphi'$. We have that $\tau_R(\ang{\{\},\varphi})=\ang{\{\},\varphi'}$, and we have that $\ang{\{\},\varphi'}\in \rulset' $ by definition of $\rulset'$. By definition of argument, we obtain $\tau_A(X)\in \args(\premprefix\nisaf')$. \par \smallskip
\noindent [Case: $X=X_1,...,X_n\Rrightarrow\varphi$] Assume, as the inductive hypothesis, that 

\begin{center}
    $\forall X_1,...,X_n \in \args(\premprefix\nisaf)$, $\tau(X_1),...,\tau(X_n)\in \args(\premprefix\nisaf')$.
\end{center}

We distinguish two subcases. First, suppose that $\ang{\{\conc(\tau(X_1)),...,\conc(\tau(X_n))\},\varphi}\in \rulset$, this implies that $\ftoprule(X)\in \rulset'$ (because all rules with premises belonging to $\rulset$ belong also to $\rulset'$ by definition), and, using the inductive hypothesis and the definition of argument, we arrive to $\tau_A(X)\in \args(\premprefix\nisaf')$.  Second, suppose that $\ang{\{\conc(\tau(X_1)),...,\conc(\tau(X_n))\},\varphi}\notin \rulset$. This implies that there are $k$ arguments, $\{Z_1,...,Z_k\} \subseteq\{X_1,...,X_n\}$ having the form $\Rrightarrow \varphi'$ 
for some $\varphi' \in c(\lanset)$. So we have $\ftoprule(X)=\ang{\{\varphi_1',...,\varphi_k',\varphi_{k+1},...,\varphi_n\},\varphi}$ with $\varphi_1, ...,\varphi_k\in \repfun$ and $\varphi_{k+1},...,\varphi_n \notin \repfun$. This implies, by definition of $\tau_R$ and $\tau_F$, that $\ftoprule(X)\in \rulset'$. The latter implies, together with inductive hypothesis and the definition of argument, that $\tau_A(X)\in \args(\premprefix\nisaf')$.
\par \medskip

\noindent [\textbf{Claim 2}. $\tau_A$ is injective] \par \smallskip

Let $X,Y \in \args(\pisaf)$ and $X\neq Y$. We want to show that $\tau_A(X)\neq \tau_A(Y)$. We proceed by induction on $X$ and $Y$. When $X$ and $Y$ have different syntactic shapes, e.g., when $X=\Rrightarrow\varphi$ and $Y=Y_1,...,Y_n\Rrightarrow\varphi$, the proof is straightforward, as $\tau_A$ clearly preserves these shapes and hence $\tau_A(X)\neq\tau_A(Y)$. We analyse the rest of the cases:\par \smallskip
[Case: $X=\varphi,Y=\psi$] We have that $X\neq Y$ by hypothesis, and that $\tau_A(X)=\tau_F(X)$ and $\tau_A(Y)=\tau_F(Y)$ by definition, so it is enough to recall that $\tau_F$ is injective to see that $\tau_A(X)\neq\tau_A(Y)$. \par 

[Case: $X=\Rrightarrow \varphi,Y=\Rrightarrow\psi$] We have that $X\neq Y$ by hypothesis, $\tau_A(X)=\Rrightarrow\tau_F(\varphi)$ and $\tau_A(Y)=\Rrightarrow\tau_F(\psi)$ by definition, so it is enough to recall that $\tau_F$ is injective.\par

[Case: $X=X_1,...,X_n\Rrightarrow \varphi,Y=Y_1,...,Y_k\Rrightarrow\psi$] Assume, as the inductive hypothesis that for all $Z_1,...,Z_{n+k}$, $Z_i\neq Z_j$ implies $\tau_A(Z_i)\neq\tau_A(Z_j)$. Note that $\tau_A(X_1,...,X_n\Rrightarrow \varphi)=\tau_A(X_1),...,\tau_A(X_n)\Rrightarrow \varphi$ and similarly for $Y$. Now, suppose that $X\neq Y$, this leads, by definition of argument, to either $\{X_1,...,X_n\}\neq\{Y_1,...,Y_k\}$ or $\varphi\neq\psi$. For the latter case, we have that $\conc(\tau_A(X))\neq \conc(\tau_A(Y))$, hence $\tau_A(X)\neq\tau_A(Y)$. For the former case, suppose reasoning towards contradiction that $\{\tau_A(X_1),...,\tau_A(X_n)\}=\{\tau_A(Y_1),...,\tau_A(Y_k)\}$. This implies that there are $X_i \in \{\tau_A(X_1),...,\tau_A(X_n)\}$ and $Y_j\in\{\tau_A(Y_1),...,\tau_A(Y_k)\}$ such that $X_i \neq Y_j$ but $\tau_A(X_i)=\tau_A(Y_j)$, but this contradicts the inductive hypothesis.

\par \medskip

\noindent[\textbf{Claim 3}. $\tau_A$ is surjective with respect to $\args(\premprefix\nisaf')$] \par \smallskip

Let $X'\in \args(\pisaf')$. We want to show that there is $X\in \args(\pisaf)$ such that $\tau_A(X)=X'$. We do this by induction of $X'$. \par
[Case: $X'=\varphi$] Then $\tau_A(X')=X'$ and we know $X'=\varphi \in \kb$, and therefore, $X'\in \args(\pisaf)$. \par 

[Case: $X'=\, \Rrightarrow \delta$] We distinguish two cases. If $\delta \notin c(\lanset)$, then by definition of $\tau_A$, we have that $\Rrightarrow \delta \in \args(\pisaf)$ and $\tau_A(\Rrightarrow \delta)=\, \Rrightarrow \delta$. If $ \delta \in c(\lanset)$, then by the definition of $\tau_A$ there is $\psi \in \lanset$ such that $\Rrightarrow \psi \in \args(\pisaf)$ and $\tau_A(\Rrightarrow \psi)=\, \Rrightarrow \delta$. \par

[Case: $X'=X_1',...,X_n'\Rrightarrow \varphi$] As the inductive hypothesis we have that there are $X_1,...,X_n \in \args(\pisaf)$ such that $\tau_A(X_1)=X_1',...,\tau_A(X_n)=X_n'$. By definition of $\tau_A$ and the inductive hypothesis we have that $\tau_A(X_1,...,X_n \Rrightarrow \varphi)=X_1',...,X_n'\Rrightarrow\varphi$. Hence we just need to show that $X= X_1,...,X_n \Rrightarrow \varphi \in \args(\pisaf) $, which amounts to showing that $r=\ang{\{\conc(X_1),...,\conc(X_n)\},\varphi} \in \rulset$. Since $X_1',...,X_n'\Rrightarrow \varphi\in \args(\pisaf')$, we have $r'=\ang{\{\conc(X_1'),...,\conc(X_n')\},\varphi}\in \rulset'$. We distinguish two subcases. First, suppose that $r'\in \rulset$, then $r'=r$ by definition of $\rulset'$ and $\tau_R$ (as otherwise we would get that there is a premise of a rule in $\rulset$ that belongs to $c(\lanset)$, which is absurd), and we are done. If $r'\notin \rulset$, then we get by definition of $\rulset'$ that there is $r_1\in \rulset$ such that $\tau_R(r_1)=r'$. We have, by definition of $\tau_R$, that $r_1=\ang{\{\tau^{-1}_F(\conc(X'_1)),...,\tau_F^{-1}(\conc(X_n))\},\varphi}$. One can show, by an easy induction on $Z\in \args(\pisaf)$, that $\tau_A(Z)=Z'$ implies $\conc(Z)=\tau^{-1}_F(\conc(Z'))$. Applying this to $r_1$, we have $r_1=\ang{\{\conc(X_1),...,\conc(X_n)\},\varphi}=r$, and therefore $r \in \rulset$.
\par \medskip

\noindent[\textbf{Claim 4}. $\big \{\ang{\tau(\args^\ast),\tau(\defrel^\ast)}\mid \comp \in \compfun(\premprefix\nisaf)\big \}=\compfun(\premprefix\nisaf')$] \par \smallskip
We will only prove the $\subseteq$-inclusion; the reverse one is completely analogous.
Let $\comp \in \compfun(\pisaf)$, which means that $\exists\kb^\ast .\kb^F\subseteq \kb^\ast\subseteq \kb^F\cup \kb^?$ such that $\naf(\ksaf)=\comp$. Let us show that $\tau(\comp)\in \compfun(\pisaf')$, which amounts to showing that there is a prem-completion of $\pisaf'$ such that its associated AF is $\tau(\comp)$. We shall prove that this prem-completion is precisely $\ang{\lanset',\contfun',\rulset',\namefun',\kb^\ast,\preceq^{\ast \prime}}$, where $\preceq^{\ast\prime}=\preceq'_{\restrict \args(\lanset',\rulset',\kb^\ast)}$. We first prove $\tau(\args(\ksaf)=\args(\ksafp)$ and later $\tau(\defrel(\ksaf))=\defrel(\ksafp)$.\par \smallskip

\noindent [\textbf{Claim 4.1.} $\tau(\args(\ksaf)\subseteq\args(\ksafp)$] Suppose that $X \in \args(\ksaf)$. We continue by induction on $X$.\par

[Case: $X=\varphi$] By definition of $\tau_A$ we have that $\tau_A(X)\in \kb^\ast$, and therefore $\tau_A(X)\in \args(\ksafp)$. \par 

[Case: $X=\Rrightarrow \varphi$] We reason by cases on $\varphi \in \repfun$. If $\varphi \in \repfun$, we have $\tau_A(X)=\Rrightarrow \varphi'$ (by definition of $\tau_A$) and $\ang{\{ \},\varphi'}\in\rulset'$ (by definition of $\pisaf'$). Which implies $\tau_A(X)\in \args(\ksafp)$ (by definition of argument). Analogously, if $ \varphi \notin \repfun$, we have that $\tau_A(X)=X$ and $\ang{\{\},\varphi}\in \rulset'$, and therefore $\tau_A(X)\in \args(\ksafp)$. \par 

[Case: $X=X_1,...,X_n \Rrightarrow \varphi$] As the inductive hypothesis, suppose that $\tau(X_1),\dots,\tau(X_n) \in \args(\ksafp)$. By definition of $\tau_A$ we have that $\tau_A(X)=\tau_A(X_1),\dots,\tau_A(X_n)\Rrightarrow \varphi$. We have that $\ftoprule(X)\in \rulset'$ (because $\tau_A:\args(\pisaf)\to \args(\pisaf')$, as we showed before). From the previous claims, the definition of argument, and the inductive hypothesis, we obtain that $\tau_A(X)\in \args(\ksafp)$. \par \smallskip

\noindent [\textbf{Claim 4.2.} $\args(\ksafp)\subseteq\tau(\args(\ksaf))$] 
Let $X' \in \args(\ksafp)$, we want to show that there is $X \in \args(\ksaf)$ such that $\tau_A(X)=X'$. We have that $\tau_A^{-1}(X') \in \args(\pisaf)$ exists and is unique because $\tau_A$ is bijective, as we have previously shown. Moreover, note that the premises of arguments are not changed by $\tau_A$, hence $\prem(X')=\prem(\tau^{-1}(X'))\subseteq \kb^\ast$. From the previous claims and the definition of argument, we can conclude that $\tau^{-1}(X') \in \args(\ksaf)$. 
\par \medskip

\noindent[\textbf{Claim 4.3.} $\tau(\defrel(\ksaf))=\defrel(\ksafp)$] 
For defeats, we will use a couple of propositions. The first one follows easily from the definition of $\preceq'$:

\begin{center}
    (P1) for every $X,Y \in \args(\pisaf)$, $X\preceq Y$ iff $\tau(X)\preceq'\tau(Y)$.
\end{center}

The second one claims that the subargument relation is preserved and antipreserved by $\tau_A$, and this can be proved by induction on $X$:

\begin{center}
    (P2) for every $X,Y \in \args(\pisaf)$, $X\in \sub(Y)$ iff $\tau_A(X)\in \sub(\tau_A(Y))$.
\end{center}

\noindent [\textbf{Claim 4.3.1.} $\tau(\defrel(\ksaf))\subseteq\defrel(\ksafp)$] Suppose that $\ang{X,Y}\in \defrel(\ksaf)$. We know that $\tau_A(X),\tau_A(Y)\in \args(\ksafp)$ by the previous items. We want to show that $\ang{\tau_A(X),\tau_A(Y)}\in \defrel(\ksafp)$. We do this by case reasoning on the definition of defeat.\par 
[Case: $X$ undercuts $Y$ (on $Z\in \sub(Y)$)] By definition of undercutting we have that:

\begin{center}
(i) $\conc(X)\in \overline{\namefun(\ftoprule(Z))}$.
\end{center}

Moreover, note that, by definition of $\namefun'$, we have.
\begin{center}
(ii) $\namefun'(\ftoprule(\tau(Z)))=\namefun(\ftoprule(Z))$.
\end{center}

Applying the identity of (ii) to (i), we obtain:
\begin{center}

(iii) $\conc(X)\in \overline{\namefun'(\ftoprule(\tau(Z)))}$.
\end{center}

We proceed by cases. First, if $\conc(\tau(X))=\conc(X)$, then we substitute identicals in (iii) and apply the definition of $\overline{\cdot}'$ to arrive at $\conc(\tau(X))\in \overline{\namefun'(\ftoprule(\tau(Z))}'$, which amounts to $\tau(X)$ undercutting $\tau(Y)$ (on $\tau(Z)$) with respect to $\ksafp$. Second, if $\conc(\tau(X))\neq \conc(X)$, this means that $X=\Rrightarrow\varphi$ with $\varphi \in \repfun$ and $\conc(\tau(X))=\varphi'$. This implies, together with (iii) and the definition of $\overline{\cdot}'$, that $\conc(\tau(X))\in \overline{\namefun'(\ftoprule(\tau(Z)))}'$, and we are done.

\par \smallskip
\noindent [Case: $X$ undermines $Y$ (on $\varphi\in \prem(Y)$) and $X\not \prec^*\varphi$] By definition of undermining we have that: 
\begin{center}
    
(H) $\conc(X)\in \overline{\varphi}$ with $\varphi \in \prem(Y)$ and $X\not \prec^*\varphi$. 

\end{center}
Note that $\varphi \in \prem(\tau_A(Y))$ (because $\tau_A$ does not change premises by definition). So all we need to show is that (i) $\conc(\tau_X(X))\in \overline{\varphi}'$ and (ii) $\tau_A(X)\not \prec^{\ast \prime} \varphi$. (ii) follows from (P1), (H) and the definition of $\preceq^{\ast \prime}$. As for (i), we consider two cases. If $\conc(X)=\conc(\tau_A(X))$, we are done (because of (H) and the definition of $\overline{\cdot}'$). If $\conc(X)\neq\conc(\tau_A(X))$, then by definition of $\tau_A$, $X=\, \Rrightarrow\delta$ for some $\delta \in \repfun $. Which implies by definition of $\tau$, that $ \conc(\tau_A(X))=\tau_F(\conc(X))$. This implies by definition of $\overline{\cdot}'$ that $\conc(\tau_A(X))\in \overline{\varphi}'$.

\par \smallskip

 \noindent [Case: $X$ contrary-undermines $Y$ (on $\varphi\in \prem(Y)$)] By definition of contrary-undermining, we have that

\begin{center}
(H) $\conc(X)\in \overline{\varphi}$ with $\varphi \in \prem(Y)$ and $\varphi \notin \overline{\conc(X)}$. 
\end{center}

Arriving at $\conc(\tau(X))\in \overline{\varphi}'$ and $\varphi \in \prem(\tau(Y))$ is as in the previous case. So we just need to show that $\varphi \notin \overline{(\conc(\tau(X))}'$. This can be done again by reasoning over cases on $\conc(X)=\conc(\tau(X))$ and applying the definition of $\overline{\cdot}'$, (H) and the definition of $\tau$.

\par \smallskip
\noindent [Case: $X$ rebuts $Y$ (on $Z$) and $X\not \prec^{\ast}Z$] By definition of rebuttal, we have the following:
\begin{center}
    
(H) $\conc(X)\in \overline{\conc(Z)}$ with $Z \in \sub(Y)$ and $\ftoprule(Z)\in\rulset_d$, and $X\not \prec^{\ast}Z$. 

\end{center}

We want to show that (i) $\conc(\tau(X))\in\overline{\conc(\tau(Z))}'$ with $\tau(Z) \in \sub(\tau(Y))$ and $\ftoprule(\tau(Z))\in\rulset'_d$ and (ii) $\tau(X)\not \prec^{\ast \prime}$. (ii) follows again from (P1), (H) and the definition of $\preceq^{\ast \prime}$. For (i), note that both the subargument relation and the type of top rule are preserved by $\tau_A$, so it is only necessary to check that $\conc(\tau(X))\in\overline{\conc(\tau(Z))}'$. This can be done reasoning by cases on whether $\conc(X)=\conc(\tau(X))$ and $\conc(Z)=\conc(\tau(Z))$ and applying the definition of $\overline{\cdot}'$. The details are omitted.

\par \smallskip
 \noindent [Case: $X$ contrary-rebuts $Y$ (on $Z$)] 

Arriving at $\conc(\tau(X))\in \overline{\conc(\tau(Z))}'$, $\tau(Z)\in\sub(\tau(Y))$ and $\ftoprule(\tau(Z))\in \rulset_d'$ is as in the previous case. So, to show that $\tau(X)$ contrary-rebuts $\tau(Y)$ (on $\tau(Z)$), we just need to show that $\conc(\tau(Z))\notin \overline{\conc(\tau(X))}'$. This is done by cases on $\conc(X)=\conc(\tau(X))$ and $\conc(Z)=\conc(\tau(Z))$. In all four cases, we need to use the case hypothesis $\conc(Z)\notin \overline{\conc(X)}$ and the definition of $\overline{\cdot}'$.

\par \medskip

\noindent [\textbf{Claim 4.3.2.} $\defrel(\ksafp)\subseteq \tau(\defrel(\ksaf))$]  Suppose that $\ang{X',Y'}\in \defrel(\ksafp)$. We know, by the previous items we have proved, that there are unique $X,Y\in \args(\ksaf)$ such that $\ang{X,Y}\in \defrel(\ksaf)$ and $\tau(X)=X'$ and $\tau(Y)=Y'$. We just need to show that $\ang{X,Y}\in \defrel(\ksaf)$. We proceed by cases on the definition of defeat:
\par \noindent
[Case: $X'$ undercuts $Y'$ on $Z'$] Recall that the existence and uniqueness of $\tau^{-1}(Z')=Z$ is also guaranteed by previous items of this theorem. Applying the definition of undercut to the case hypothesis, we have the following.
\begin{center}
    (H) \qquad$\conc(X')\in \overline{\namefun'(\ftoprule(Z'))}'$ with $Z'\in \sub(Y')$.
\end{center}

We already know that $Z \in \sub(Y)$ (by (P2)). We proceed by cases on $\conc(X')=\conc(X)$ and $Z=Z'$. Let us just explain the first case. If both $\conc(X')=\conc(X)$ and $Z=Z'$, we substitute identicals in (H) and obtain  $\conc(X)\in \overline{\namefun'(\ftoprule(Z))}'$. Using (H) again, and applying the definitions of $\namefun'$ and $\overline{\cdot}'$, we get to $\conc(X)\in \overline{\namefun(Z)}$, which together with $Z \in \sub(Y)$ means that $X$ undercuts $Y$. For cases in which $Z\neq Z'$ or $\conc(X')\neq\conc(X)$ we need to check that the definitions of $\namefun'$ and $\overline{\cdot}'$ force $X$ to undercut $Y$ in $Z$. We omit details.
\par \smallskip
[Case: $X'$ undermines $Y'$ on $\varphi$ and $X' \not \prec^{\ast\prime} \varphi$] Applying the definition of undermining, we have that $\conc(X')\in \overline{\varphi}'$ with $\varphi\in \prem(Y')$ and $X' \not \prec^{\ast\prime} \varphi$. $X' \not \prec^{\ast} \varphi$ is guaranteed by (P1). So we just need to show that $\conc(X)\in \overline{\varphi}$. If $\conc(X)=\conc(X')$, we apply the definition of $\overline{\cdot}'$ and we are done. If $\conc(X)\neq\conc(X')$, then $\conc(X)=\tau^{-1}_F(\conc(X))$ (we know that $X$ is a simple premiseless argument) and the definition of $\overline{\cdot}'$ guarantees again that $\conc(X)\in \overline{\varphi}$.
\par \smallskip

Contrary-underminings and (contrary)-rebuttals follow the same reasoning patterns as their corresponding cases in the proof of Claim 4.3.1. We omit further details.
\end{proof}
%=====================================================

After this, it is possible to show that any rul-ISAF can be translated into an equivalent tidy prem-ISAF, thereby establishing the following theorem.

\begin{theorem}\label{thm:rule-more-prem}
    $\text{prem-ISAFs } \lexpressive \text{ rul-ISAFs.}$ 
\end{theorem}

\newcommand{\risaf}{\rulprefix\nisaf}
\newcommand{\unravelrisaf}{\ang{\lanset,\overline{\cdot},\rulset',\namefun',\kb',\preceq'}}

\begin{proof}[Proof (sketch)]
    Let $\pisaf=\unravelsaf$ with $\rulset=\rulset_s\cup\rulset_d$ and $\kb=\kb_n^F\cup\kb_n^?\cup\kb_p^F\cup\kb_p^?$. We can assume without loss of generality, because of Lemma \ref{lemma:tidyasexpressive}, that $\pisaf$ is tidy (if it weren't, we would just transform it into an equivalent, tidy prem-ISAF).
    
     Let us inductively define the following translation function for arguments in $\args(\pisaf)$:

    \begin{itemize}
        \item $\tau(\varphi)= \left\{\begin{array}{lr}
   \sto\varphi  &\text{if } \varphi \in \kb^?_n
         \\
        \Rightarrow \varphi &\text{if } \varphi\in \kb^?_p \\
        \varphi & \text{otherwise}
        \end{array}\right.$
\item $\tau(\Rrightarrow\varphi)=\, \Rrightarrow\varphi$ for $\Rrightarrow \in \{\sto,\Rightarrow\}$.
        \item $\tau(X_1,...,X_n\Rrightarrow\varphi)=\tau(X_1),...,\tau(X_n)\Rrightarrow\varphi$.
    \end{itemize}
    
    Now, we define an equivalent rul-ISAF as follows, $\rulprefix\nisaf=\ang{\lanset,\overline{\cdot},\rulset',\namefun,\kb',\preceq'}$, where:

    \begin{itemize}
        \item $\rulset'=\rulset_s^{\fix }\cup\rulset_s^{?}\cup \rulset_d^{\fix}\cup\rulset_d^{?}$, where:

        \begin{itemize}
            \item $\rulset_s^{\fix }=\rulset_s$.
            \item  $\rulset_s^{?}=\{\ang{ \{\},\varphi}\mid \varphi \in \kb_n^?\}$.
        \item $\rulset_d^{\fix }=\rulset_d$.
            \item  $\rulset_d^{?}=\{\ang{ \{\},\varphi}\mid \varphi \in \kb_p^?\}$.
        \end{itemize}

        \item $\kb'=\kb\setminus\kb^?$.
    \item $\preceq'= \{\ang{\tau(X),\tau(Y)}\mid X \preceq Y\}$.
    
    \end{itemize}

    The proof follows the same structure as that of Lemma \ref{lemma:tidyasexpressive}. We sketch its structure (the main claims to be proved) and point out the critical points.\par \medskip

    \noindent [\textbf{Claim:} For all $X \in \args(\pisaf)$, $\tau(X)\in\args(\risaf)$] This is proved by induction on $X$. \par \smallskip 

    \noindent [\textbf{Claim:} $\tau$ is injective, that is, for all $X,Y\in \args(\pisaf)$, $X\neq Y$ implies $\tau(X)\neq\tau(Y)$] This is done by simultaneous induction on $X$ and $Y$. We have nine cases on the possible syntactic shapes of $X$ and $Y$ (i.e., $X=\varphi$ and $Y=\psi; X=\, \Rrightarrow \varphi$ and $Y=\psi$, etc.). The inductive hypothesis is only needed for the case where $X=X_1,...,X_n\Rrightarrow \varphi$ and $Y=Y_1,...,Y_k \Rightarrow \psi$. Importantly, the case $X=\varphi$ and $Y=\,\Rrightarrow \psi$ (and the symmetric one, $X=\Rrightarrow \varphi$ and $Y=\psi$) needs the assumption that $\pisaf$ is tidy, which we justified at the beginning of this proof. \par \smallskip

    \noindent [\textbf{Claim:} $\tau$ is surjective (with respect to $\args(\risaf)$] By induction on $X'\in \args(\risaf)$. \par \smallskip
\par \medskip
    The next step is showing that $\pisaf$ and $\risaf$ are indeed equivalent. That is, we need to show that $\{\tau(\comp)\mid \comp\in\compfun(\pisaf)\}=\compfun(\risaf)$. For this, we establish a one-to-one correspondence between the prem-completions (of $\pisaf$) and the rul-completions (of $\risaf$). Let $\kb^\ast=\kb^\fix \cup \kb^{\ast ?}_n\cup \kb^{\ast ?}_p$ be the set of premises of a prem-completion; its corresponding set of rules in the rul-completion of $\risaf$ is $\rulset^\ast=\rulset^\fix\cup\rulset^{\ast ?}_s\cup \rulset^{\ast ?}_d$ where:

    \begin{itemize}
        \item $\rulset^{\ast ?}_s=\{\ang{\{\},\varphi}\mid \varphi \in \kb^\ast_n\}$; and
        \item $\rulset^{\ast ?}_d=\{\ang{\{\},\varphi}\mid \varphi \in \kb^\ast_p\}$.
    \end{itemize}

\newcommand{\precomp}{\ang{\lanset,\contfun,\rulset,\namefun,\kb^\ast,\preceq^\ast}}
\newcommand{\rulcomp}{\ang{\lanset,\contfun,\rulset^\ast,\namefun,\kb',\preceq'}}
The definition from $\rulset^\ast$ to $\kb^\ast$ is analogous. Finally, we prove that for every prem-completion $\ang{\lanset,\contfun,\rulset,\namefun,\kb^\ast,\preceq}$, $\tau(\naf(\ang{\lanset,\contfun,\rulset,\namefun,\kb^\ast,\preceq})=\naf(\ang{\lanset,\contfun,\rulset^\ast,\namefun,\kb',\preceq'})=$. This is split into four claims. \par \smallskip
\noindent [\textbf{Claim:} $\tau(\args(\precomp))\subseteq\args(\rulcomp)$] By induction on $X \in \args(\precomp)$. \par \smallskip

\noindent [\textbf{Claim:} $\args(\rulcomp)\subseteq\tau(\args(\precomp))$] By induction on $X'\in \args(\rulcomp)$. \par 
\smallskip

\noindent [\textbf{Claim:} $\tau(\defrel(\precomp))\subseteq\defrel(\rulcomp)$] By cases on the definition of defeat ($X$ defeats $Y$ on $Z$). In the cases where $X$ (contrary-)undermines $Y$ (on $\varphi$) with respect to $\defrel(\precomp)$, we distinguish two subcases: $\tau(\varphi)\in \prem(\tau(Y)) $, which leads to $\tau(X)$ (contrary-)undermines $\tau(Y)$; and $\tau(\varphi)\notin \prem(\tau(Y))$, which leads to $\tau(\varphi)=\Rightarrow \varphi \in \sub(\tau(Y))$ and, subsequently, to a (contrary-)rebuttal.

\par \smallskip
\noindent [\textbf{Claim:} $\defrel(\rulcomp)\subseteq\defrel(\args^\ast(\precomp))$] By cases on the definition of defeat. The critical case here is (contrary-)rebuttal, where we analyse two subcases again: whether the targeted defeasible rule is certain or uncertain. In the former, we obtain that $\tau^{-1}(X)$ also (contrary-)rebuts $\tau^{-1}(Y)$. In the latter,  we find that $\tau^{-1}(X)$ (contrary-)undermines $\tau^{-1}(Y)$. \par 
\end{proof}

\begin{remark}[On the scope of the previous theorem] The previous proof makes use of a rather strong and non-standard assumption: our ASPIC$^{+}$ frameworks admit \textit{any} preference relation among arguments. Although this is indeed the case in some presentations of the literature, preferences in ASPIC$^{+}$ are usually assumed to be the result of some lifting principles among the defeasible elements (rules and premises) of the involved arguments. Two of the most popular of these principles are the weakest-link and the last-link principle \cite{modgil2013aspic}. Importantly, the construction used in the proof does not work for either of these two principles (counterexamples can be easily found). On the positive side, we have that the proof runs without problems whenever the arguments of the involved ISAF are equally strong (the preference relation is either empty or total). We decided to leave the proof as it stands to show that fully abstract preferences can actually be handled and because we think it paves the way towards the adaptation of the result to different lifting principles.\footnote{We thank an anonymous reviewer of JAIR for pointing this out.}
\end{remark}

\subsection{Summary of Results}

The following diagram summarises what we have proved about the relative expressivity of some of the formalisms defined up until now. An arrow from a formalism $\mathcal X$ to $\mathcal Y$ expresses that $\mathcal{X} \lexpressive \mathcal{Y}$. Reflexive and transitive arrows are omitted for readability. 
Note that non-depicted arrows represent directions that have been shown not to hold (or that are easily provable).

\begin{center}

    \begin{tikzpicture}[
    box/.style={
        draw,
        rounded corners=8pt,
        thick,
        inner sep=8pt,
        minimum height=1cm,
        minimum width=1.4cm,
        align=center,
        font=\sffamily\small
    },
    ->,
    thick,
    node distance=1.5cm
]

    % Nodes
    \node[box] (arg) {arg-IAFs};
    \node[box, right=of arg] (prem) {prem-ISAFs};
    \node[box, right=of prem] (rul) {rul-ISAFs};
    \node[box, right=of rul] (imp) {imp-arg-\\IAFs};
    \node[box, right=of imp] (dep) {dep-arg-\\IAFs};

    % Arrows
\draw (arg) -> node[above, font=\footnotesize] {Th.\ \ref{thm:arg-IAFs-lexpressive-prem-ISAFs} and \ref{prop:prem-isafs-negative}} (prem);
    \draw (prem) -- node[above, font=\footnotesize] {Th.\ \ref{prop:prem-no-more-rule} and \ref{thm:rule-more-prem}} (rul);
    \draw (rul) -- node[above, font=\footnotesize] {Th.\ \ref{thm:rul-iafs-representation} and \ref{prop:impargiafs-not-lessexpresive-rulisafs}} (imp);
    \draw (imp) -- node[above, font=\footnotesize] {Prop.\ \ref{prop:dep-arg-iaf}} node[below, font=\footnotesize]  {and Fact \ref{fact:imp-less-expressive-dep}} (dep);

\end{tikzpicture}
 \captionof{figure}{Expressivity results for rul-ISAFs and prem-ISAFs}\label{fig:arg-ISAFs-results}
\end{center}

\begin{remark}
   It is not difficult to show that 
  neither disjunctive arg-IAFs (only allowing $\opor$-dependencies) nor nand-arg-IAFs (only allowing $\opnand$-dependencies) are expressive enough to capture prem-ISAFs or rul-ISAFs. 
\end{remark}

\begin{remark}
 Not all theorems proved in this section are independent of each other. For instance, it would be enough to prove the depicted arrows and the absence of their reversal in Figure \ref{fig:arg-ISAFs-results}, and the rest of the results (the non-depicted arrows) follow from transitivity of $\lexpressive$. However, we find the proofs of each of these results interesting due to the constructions they are based on. 
\end{remark}

%==============
\section{Other Structured Argumentation Frameworks with Uncertain Components}\label{sec:other-frameworks}
%==============

Besides ASPIC$^{+}$, two general frameworks for structured argumentation can be lifted to the abstract level, where incomplete counterparts can be defined along the same lines, namely ABA frameworks and deductive argumentation. In fact, the ASPIC$^{+}$  notation allows for encoding both formalisms. However, a number of differences in the way these formalisms define the notions of argument and attack prevent a straightforward translation of these systems as instances of ASPIC$^{+}$.
What is more, these differences force a different lifting to the level of abstract argumentation and therefore do not allow the direct application of all the results of the previous section. For this reason, here we will only hint at how to define incompleteness in ABA and deductive argumentation, indicating which results are easily transferable from ASPIC$^+$ and which are not (and why).
In what follows, we discuss these formalisms separately.

%====================
\subsection{Incompleteness in ABA}
%====================
An ABA framework can be identified as a special type of SAF with no naming function, no preferences and a restricted contrariness function. More formally, the definition runs as follows.

\begin{definition}[\citeauthor{toni2014tutorial}, \citeyear{toni2014tutorial}]
    
An \textbf{assumption-based argumentation framework} (ABAF) is a tuple $\abaf=\unravelabaf$ where:

\begin{itemize}
\item $\lanset$ is a formal language;
\item $\rulset  \subseteq \wp_{fin}(\lanset)\times \lanset$ is a set of inference rules, i.e., pairs of finite sets of formulas and formulas;
\item $\asset\subseteq \lanset$ is a set of assumptions; and
\item $\contfun:\asset \to  \lanset$ is a (total) contrariness function.
\end{itemize}
 
\end{definition}

In comparison to ASPIC$^+$, there are some salient differences:

\begin{itemize}
\item In ABA, there is only one type of rules, corresponding to strict rules. As a consequence, no rebuttal attacks can be defined. Rules with empty body are called ``facts'' in the ABA literature.

\item ABA features only one type of premises (assumptions) corresponding to ordinary premises in ASPIC$^+$ (i.e., there are no axioms).

\item  The contrariness function is defined only for assumptions in ABA, not for the whole language. Moreover, the uniqueness of the contrary of a formula is assumed in ABA (as opposed to ASPIC$^+$, where the range of $\overline{\cdot}$ is $\wp(\lanset)$). Note that there are extensions of ABA where this assumption is dropped (e.g., \citeauthor{fan2014}, \citeyear{fan2014}). We do not consider them here.

\item  The naming function $\namefun$ is missing in ABA. As a consequence, no undercutting attacks can be defined.
\end{itemize}
According to the first and fourth items, ABA can only feature undermining attacks, i.e. attacks on assumptions.
Further, in the standard variant of ABA, there is no explicit modelling of a preference relation among arguments.
There are proposals filling this gap (e.g., \cite{CyrasToni16}, but we leave them out of our analysis.
What is more, the very same notion of argument is different in ABA, and in fact, there is more than one in the literature. An argument can be understood either as a set of assumptions or else as a derivation (see, e.g., \citeauthor{toni2014tutorial}, \citeyear{toni2014tutorial}). As a further complication, there are two different approaches for defining derivations: one based on a \emph{tree-style} construction and the other being \textit{inductive} (either in a \emph{forward fashion} or in a \textit{backward} one). The tree-style and the backward inductive style can be proven to be fully equivalent. In contrast, the forward inductive definition is only equivalent to the other two modulo some semantic assumptions (see \cite{toni2014tutorial} for details).

For the sake of our discussion, we concentrate on ABA with arguments as derivations, where derivations are defined as trees (i.e., backward inductions). Moreover, we adapt the definition to the ASPIC$^+$ notation.

\begin{definition}\label{def:aba-argument}
The set of (backward inductive derivation-based) \textbf{arguments of a given ABAF} $\abaf=\unravelabaf$, denoted $\args(\unravelabaf)$, is defined inductively. Together with the notion of argument, we define some auxiliary functions: $\sub(\cdot)$ (returns the \textbf{subarguments} of an argument),\footnote{The notion of subargument is not usually taken into account in the ABA literature. However, it is needed for us to transfer the expressivity results from the previous section.} $\prem(\cdot)$ (returns the premises or \textbf{assumptions} of an argument), and $\conc(\cdot)$ (returns the \textbf{conclusion} of an argument). We establish that $A \in \args(\unravelabaf)$ iff $A$ is any finite expression satisfying one of the following conditions:

\begin{itemize}
\item $A=[\varphi]$ if $\varphi \in \kb$, with 
\begin{itemize}
    \item 
$\prem(A)=\conc(A)=\{\varphi\}$, and
\item 
$\sub(A)=\{[\varphi]\}$.

\end{itemize}
\item $A=[A_1,...,A_n \sto \varphi]$ (with $n\geq 0$) if $A_1,...,A_n$ are arguments and $\ang{\{\conc(A_1),...,\conc(A_n)\},\varphi} \in \rulset$, with 

\begin{itemize}
\item $\prem(A)=\prem(A_1)\cup...\cup\prem(A_n)$, 

\item $\conc(A)=\varphi$, and

\item $\sub(A)=\{A\}\cup \sub(A_1)\cup...\cup\,\sub(A_n)$. 

\end{itemize}
\end{itemize}
Squared brackets are omitted when they do not confuse. An argument whose conclusion is $\varphi \in \lanset$, whose assumptions are $S \subseteq \kb$ and uses all rules in $R\subseteq \rulset$ in its construction, is sometimes abbreviated as $S\vdash_{R}\varphi$ in the specialised literature. Moreover, we drop the subscript from $\vdash_R$ whenever $R$ is clear from context.
\end{definition}

This provides all the elements needed to define the abstract AF associated with a given ABAF.

\begin{definition}[\citeauthor{toni2014tutorial}, \citeyear{toni2014tutorial}]
Given an ABAF $\abaf=\unravelabaf$, its \textbf{associated AF} is $\daf(\abaf)=\tuple{\args,\defrel}$ where:

\begin{itemize}
\item $\args=\args(\unravelabaf)$; 
\item $\defrel\subseteq \args \times \args$ is given by $\tuple{S\vdash \varphi,T\vdash \sigma}\in \defrel$ iff $\overline{\delta}=\varphi$ for some $\delta \in T$. 
\end{itemize}
\end{definition}

\begin{remark}
The abbreviation $S\vdash_{R}\varphi$ should not be used when there is more than one derivation from $S$ to $\varphi$ using $R$, since each derivation counts as a different argument in the ABA definition of argument-as-derivation (Def. \ref{def:aba-argument} and \cite[Section 4]{toni2014tutorial}) and therefore all of them must be included as different entities in the associated AF. Analogously, the abbreviation $S\vdash \varphi$ shall not be used when we have derivations from $S$ to $\varphi$ using different sets of rules. See \cite[Section 4]{lehtonen2023argumentation} for a more detailed explanation of different ways for representing derivation-based arguments in ABA.
\end{remark}

As for ASPIC$^{+}$, ABA also allows two forms of uncertainty, about inference rules and about assumptions. The following definition captures the first option.

\begin{definition}
    A \textbf{rule-incomplete ABAF} (rul-IABAF) is a tuple $\rulprefix\iabaf=\unravelabaf$ where every component is defined as in an ABAF except for $\rulset$, which is split into two disjoint subsets $\rulset^\fix$ and $ \rulset^?$ representing, respectively, certain and uncertain rules. A \textbf{rule-completion} of $\iabaf=\unravelabaf$ is any $\abaf^{\ast}=\tuple{\lanset,\rulset^{\ast},\asset,\contfun}$ where $\rulset^{\ast}$ is any set such that\ $\rulset^{\fix}\subseteq \rulset^{\ast}\subseteq \rulset^{\fix}\cup \rulset^{?}$. The set of all rule-completions of $\rulprefix\iabaf$ is denoted as $\rulprefix\compfun(\iabaf)$.
    \end{definition}

Assumption incompleteness is defined in the same way as follows.
\begin{definition}
    An \textbf{assumption-incomplete ABAF} (asu-IABAF) is a tuple $\assprefix\iabaf=\unravelabaf$ where every component is defined as in an ABAF except for $\asset$, which is split into two disjoint subsets $\asset^\fix$ and $ \asset^?$, representing, respectively, certain and uncertain assumptions. An \textbf{assumption-completion} of $\iabaf$ is any $\abaf^{\ast}=\tuple{\lanset,\rulset,\asset^{\ast},\contfun}$ where $\asset^{\ast}$ is any set such that\ $\asset^{\fix}\subseteq \asset^{\ast}\subseteq \asset^{\fix}\cup \asset^{?}$. The set of all assumption-completions of $\assprefix\iabaf$ is denoted as $\assprefix\compfun(\iabaf)$.
\end{definition}

As for ASPIC$^{+}$, the completions of a rul-IABAF or of an asu-IABAF correspond to the abstract counterparts of their completions and are defined along the same lines.

\begin{definition}
    
The set of \textbf{completions} of a given $\rulprefix\iabaf$ is:
{\small
$$\compfun(\rulprefix\iabaf)=\{\daf(\abaf^{\ast})\mid\abaf^{\ast} \in \rulprefix\compfun(\iabaf) \}\text{.}$$
}

\noindent Likewise, the set of \textbf{completions} of a given $\assprefix\iabaf$ is:
\small
$$\compfun(\assprefix\iabaf)=\{\daf(\abaf^{\ast})\mid\abaf^{\ast} \in \assprefix\compfun(\iabaf) \}\text{.}$$
\normalsize

\end{definition}

In perfect analogy with ASPIC$^{+}$, it is possible to show that both
rul-IBAFs and asu-IABAFs may generate uncertain arguments but cannot generate uncertain defeats. This is provided by the next two results that we state without proof. 

\begin{proposition}\label{rul-IABAF-der-arg-inc}
    Let $\rulprefix\iabaf$ be a rul-IABAF. Let $\ang{\args_1,\defrel_1},\ang{\args_2,\defrel_2} \in \compfun(\rulprefix\iabaf)$. Then, for every $X,Y \in \args_1\cap \args_2$, $\ang{X,Y} \in \defrel_1$ iff $\ang{X,Y}\in \defrel_2$.
\end{proposition}

\begin{proposition}\label{asu-IABAF-der-arg-inc}
    Let $\assprefix\iabaf$ be a asu-IABAF. Let $\ang{\args_1,\defrel_1},\ang{\args_2,\defrel_2} \in \compfun(\assprefix\iabaf)$. Then, for every $X,Y \in \args_1\cap \args_2$, $\ang{X,Y} \in \defrel_1$ iff $\ang{X,Y}\in \defrel_2$.
\end{proposition}

Also in this context, it is possible to show that arg-IAFs are not expressive enough to capture rule-incomplete ABAFs.
\begin{theorem}\label{prop:negative-rul-der-IABAFs}
    rul-IABAFs $\not \lexpressive$ arg-IAFs.
   \footnote{Here, rul-IABAFs denotes the class of all rul-IABAFs together with the notion of argument and completion that we have given above.}
\end{theorem}

\begin{proof}
    Same as Theorem \ref{thm:rul-ISAFs-n-less-expressive-arg-IAFs}.
\end{proof}

And the same holds for assumption-incomplete ABAFs:

\begin{theorem}\label{prop:negative-asu-der-IABAFs}
    asu-IABAFs $\not \lexpressive$ arg-IAFs.
\end{theorem}
  
\begin{proof}
    Same as Theorem \ref{prop:prem-isafs-negative}.
\end{proof}

Here, as well, it is possible to prove that argumentation frameworks with implicative dependencies are powerful enough to capture (derivation-based) rule-incomplete ABAFs:

\begin{theorem}\label{thm:rul-IABAs-less-imp-arg}
rul-IABAFs $ \lexpressive$ imp-arg-IAFs.
\end{theorem}

\begin{proof}
    Same as Theorem \ref{thm:rul-iafs-representation}.
\end{proof}

As for ASPIC$^{+}$ the opposite does not hold:

\begin{theorem}\label{thm:imp-arg-IAFs-not-less-rul-IABAFs} 
    imp-arg-IAFs $\not \lexpressive$ rul-IABAFs
\end{theorem}

\begin{proof}
    Same as Theorem \ref{prop:impargiafs-not-lessexpresive-rulisafs}.
\end{proof}

Therefore, rul-IABAFs behave as rul-ISAFs about their lower and upper bounds of expressivity.

With the techniques employed in the ASPIC$^{+}$ case, it is also possible to show that implicative dependencies are enough to express assumption-incomplete IABAFs: 

\begin{theorem}\label{thm:asu-iafs-representation}
   asu-IABAFs $\lexpressive$ imp-arg-IAFs.
\end{theorem}

\begin{proof}
   Same as Theorem \ref{thm:prem-iafs-representation}.
\end{proof}

Unfortunately, there is not much more to say about asu-IABAFs, rul-IABAFs, and how they compare to each other. The missing expected results -- namely $\text{arg-IAFs}\lexpressive\text{rul-IABAFs}$, $\text{arg-IAFs}\lexpressive\text{asu-IABAFs}$, { $\text{imp-arg-IAFs }\not \lexpressive\text{asu-IABAFs}$} and $\text{asu-IABAFs}\prec\text{rul-IABAFs}$ -- cannot be directly imported from their counterparts in ASPIC$^{+}$. The main reason for this is that the corresponding proofs play with the fact that the contrariness function in ASPIC$^+$ takes $\wp(\lanset)$ (instead of $\lanset$) as a codomain. The results and { open problems} of this subsection are summarised in Figure \ref{fig:IABAFs-results}.

\begin{center}
   \begin{tikzpicture}[
    box/.style={
        draw,
        rounded corners=8pt,
        thick,
        inner sep=8pt,
        minimum height=1cm,
        minimum width=1.4cm,
        align=center,
        font=\sffamily\small
    },
    ->,
    thick,
    node distance=1.5cm and 2cm
]

    % Nodes
    \node[box] (arg) {arg-IAFs};
    \node[box, above right=0.8cm and 1.8cm of arg] (asu) {asu-IABAFs};
    \node[box, below right=0.8cm and 1.8cm of arg] (rul) {rul-IABAFs};
    \node[box, below right=0.8cm and 1.8cm of asu] (imp) {imp-arg-IAFs};

    % Arrows
    \draw[dashed] (arg) -- node[above left=2pt] {Th.\ \ref{prop:negative-asu-der-IABAFs}} (asu); 
    \draw[dashed] (arg) -- node[below left=2pt] {Th.\ \ref{prop:negative-rul-der-IABAFs}} (rul);
    
    \draw[dashed] (rul) to[bend left=20] (asu);
    \draw[dashed] (asu) to[bend left=20] (rul);
   
    \draw (asu) -- node[above right=2pt] {Th.\ \ref{thm:asu-iafs-representation}} (imp); 
    \draw[dashed] (imp) to [bend left=20] (asu); 
    \draw (rul) --  node[below right=2pt] {Th.\ \ref{thm:rul-IABAs-less-imp-arg} and \ref{thm:imp-arg-IAFs-not-less-rul-IABAFs}}  (imp);

\end{tikzpicture}
    \captionof{figure}{Expressivity results for rul-IABAFs and asu-IABAFs. Reflexive arrows are omitted. Dashed arrows represent 
{unknowns about the $\lexpressive$ relation}. Missing arrows indicate that the $\lexpressive$ relation does not hold.}\label{fig:IABAFs-results}
\end{center}

%==================
\subsection{Incompleteness in Deductive Argumentation}\label{sec:DA-with-uncertainty}
%==================

Deductive argumentation \cite{besnard2014constructing,besnard2018review} is the third rule-based formalism for structured argumentation whose semantics are expressed in terms of abstract AFs. Different from ASPIC$^{+}$ and the version of ABA that we have used, deductive argumentation does not explicitly use tree-based arguments.
Its main idea is to select, depending on the specific application, a base deductive logic from which one can construct arguments according to its consequence relation. Typical such choices are, among others, simple logic, classical logic, modal logics or paraconsistent logics. 

\par

As for ABA, a Deductive Argumentation framework can be regarded as a special type of SAF with no naming function and no preferences. 
\begin{definition}
    
A \textbf{deductive argumentation framework} (DAF) is a tuple $\ddaf=\unravelabaf$ where:

\begin{itemize}
\item $\lanset$ is a formal language;
\item $\rulset  \subseteq \wp_{fin}(\lanset)\times \lanset$ is the set of inference rules (i.e. pairs of sets of formulas and formulas) of a specific logical calculus;
\item $\asset \subseteq \lanset$ is a set of formulas of the language; 
\item $\contfun:\lanset \to  \wp(\lanset)$ is a (total) contrariness function.
\end{itemize}
 
\end{definition}

In comparison to ASPIC$^+$ and ABA, its peculiarities are:

\begin{itemize}
\item The set of inference rules is encoded by a specifically selected calculus (e.g. propositional, first-order, or modal logic) and is typically infinite. Further, rules are all of the same type (corresponding to strict ones), and therefore, there is no possible uncertainty about them.

\item There is only one type of assumption, consisting of (arbitrary) formulas of the language.

\item  Contrariness is defined for any formula of the language (is a total function). Its range is $\wp(\lanset)$) but typically corresponds to the \emph{complement} of a given formula, i.e. the set of formulas that are logically equivalent to its negation.

\item  The naming function $\namefun$ is missing as in ABA. 

\end{itemize}
As a consequence of the fourth item, Deductive argumentation has only undermining and rebuttal attacks.
As for ABA, there is no preference relation among arguments. 

\par 
\medskip

In the case of deductive argumentation, the notion of argument is defined after the notion of derivation/consequence $\vdash_L$ of the selected logic $L$. An argument can be understood as a \emph{minimal derivation} from $\asset$, i.e. An \textbf{argument} (with respect to $\asset$) is a pair $\ang{\Gamma,\alpha}$ with $\Gamma \subseteq \asset$ such that $\Gamma\vdash_L \alpha$ and $\nexists \Gamma' \subset \Gamma$ such that $\Gamma'\vdash_L \alpha$.

As mentioned, there are several instances of deductive argumentation, depending on the choice of the underlying logic. For our present purposes, we discuss the case of simple logic, which displays a behaviour similar to that of ABA.
 \par 

\begin{definition}[\citeauthor{besnard2014constructing}, \citeyear{besnard2014constructing}]
    
The \textbf{language of simple logic} $\lanset_s$ consists of a set of \textbf{literals} $L=\{p,\lnot p, q, \lnot q,...\}$ (propositional atoms or their negations) and a set of \textbf{simple rules}, i.e., formulas of the form $\alpha_1\land \dots \land \alpha_n \to \beta$, where each $\alpha_i$ and $\beta$ are literals. For each literal $\alpha \in L$ we define its \textbf{complement} $\sim \alpha$ as follows: $\sim \alpha=p$ if $\alpha=\lnot p$ for some atom $p$, and $\sim \alpha=\lnot \alpha$ otherwise (note that for every two literals $\alpha$ and $\beta$ it holds that $\alpha=\sim \beta$ iff $\sim \alpha=\beta$). A \textbf{simple knowledge base} $\asset$ is just a subset of the language $\lanset_s$. The \textbf{simple consequence relation} (with respect to $\asset$) is the smallest relation $\vdash_s \subseteq \wp(\lanset_s)\times \lanset_s$ such that $\Gamma\vdash_s \beta$ iff $\alpha_1\land \dots \land \alpha_n \to \beta \in \asset$ and for every $i$ such that $1 \leq i \leq n$, either $\alpha_i \in \asset$ or $\asset\vdash_s \alpha_i$. A \textbf{simple argument} (with respect to $\asset$) is a pair $\ang{\Gamma,\alpha}$ such that $\Gamma\subseteq \kb$, $\alpha \in \lanset$ $\Gamma\vdash_s \alpha$ and $\nexists \Gamma' \subset \Gamma$ such that $\Gamma'\vdash_s \alpha$. For each argument $\ang{\Gamma,\alpha}$ we define $\prem(\ang{\Gamma,\alpha})=\Gamma$ and $\conc(\ang{\Gamma,\alpha})=\alpha$.\footnote{Called \textit{support} and \textit{claim}, respectively, in the specialised literature \cite{besnard2014constructing}.} Given two simple arguments $A$ and $B$ we say that there is a \textbf{simple undermining} from $A$ to $B$ iff $\conc(A)=\sim \alpha_i$ for some $\alpha_i \in \{\alpha_1,....,\alpha_n\}$ with $\alpha_1 \land \dots \land \alpha_n \to \beta \in \prem(B)$.\footnote{This form of defeat is called \textit{simple undercut} in the specialised literature \cite{besnard2014constructing}. We changed the terminology to avoid confusion with the one used for ASPIC$^+$ above. Moreover, a form of simple rebuttal can be considered. However, it does not affect our discussion, so we decided to leave it out for simplicity.} Finally, the \textbf{AF associated to simple knowledge base $\asset$} is the pair $\naf(\asset)=\ang{\args(\asset), \defrel(\asset)}$ where $\args(\asset)$ is the set of all simple arguments (with respect to $\asset$) and $\defrel(\asset)$ are all simple underminings among the elements of $\args(\asset)$.\footnote{A form of simple rebuttal can also be taken into account (see \cite{besnard2014constructing}). Still, it is irrelevant to our purposes here.} 
\end{definition}

Uncertainty about inference rules is not an option in deductive argumentation, since they are provided in bulk by the underlying logic. Instead, uncertainty about the knowledge base can be introduced in the obvious way as follows. 

\begin{definition}
    An \textbf{incomplete simple knowledge base} (ISLAF) is just a simple knowledge base that is split into two disjoint parts $\asset=\asset^\fix \cup \asset^?$, the certain or fixed part ($\asset^\fix$) and the uncertain one ($\asset^?$). The set of \textbf{completions of $\asset$} is defined as
\begin{center}
$\compfun(\asset)=\{\naf(\asset^\ast)\mid \asset^\fix \subseteq \asset^\ast \subseteq \asset^\fix \cup \asset^?\}$.
\end{center}

\end{definition}

As for previous formalisms, uncertain knowledge bases generate only argument uncertainty at the abstract level.

\begin{proposition}\label{prop:simple-logic-arg-incompleteness}
   Let $\asset$ be a incomplete simple knowledge base, and let $\ang{\args_1,\defrel_1},\ang{\args_2,\defrel_2}\in \compfun(\asset)$. Then, for every $X,Y \in \args_1\cap \args_2$, $\ang{X,Y} \in \defrel_1$ iff $\ang{X,Y}\in \defrel_2$.
\end{proposition}

Again, standard arg-IAFs are not expressive enough to capture the kind of uncertainty generated by uncertain knowledge bases for simple logic.
\par \medskip

\begin{theorem}\label{prop:negative-simple-logic-arguments}
ISLAFs $\not \lexpressive$ arg-IAFs.
\end{theorem}

\begin{proof}
    Same as Theorem \ref{prop:prem-isafs-negative}.
\end{proof}

 As for ASPIC$^{+}$ and ABA, it turns out that adding implicative dependencies is enough to capture ISLAFs:

\begin{theorem}\label{thm:ISLAFs-lessexpressive-imp-arg-IAFs}
    
$\text{ISLAFs }\lexpressive \text{ imp-arg-IAFs}$.

\end{theorem}

\begin{proof}
   Same as Theorem \ref{thm:prem-iafs-representation}.
\end{proof}

But the reverse does not hold.

\begin{theorem}\label{thm:negative-ISLAFs}
$\text{imp-arg-IAFs } \not \lexpressive \text{ ISLAFs}$.
\end{theorem}

\begin{proof}
    Same as Theorem \ref{thm:imp-strictly-more-prem}.
\end{proof}

Once again, the expected result $\text{arg-IAFs } \lexpressive \text{ ISLAFs}$ cannot be proven via a direct application of the argument given for Theorem \ref{thm:arg-IAFs-lexpressive-prem-ISAFs}, because the simple negation $\sim$ cannot capture ASPIC's contrariness function on its full generality. Expressivity results and open problems about ISLAFs are summarised in Figure \ref{fig:ISLAFs-results}.
 
\par \medskip
\begin{center}
    \begin{tikzpicture}[
    box/.style={
        draw,
        rounded corners=8pt,
        thick,
        inner sep=8pt,
        minimum height=1cm,
        minimum width=1.4cm,
        align=center,
        font=\sffamily\small
    },
    ->,
    thick,
    node distance=1.5cm
]

    % Nodes
    \node[box] (arg) {arg-IAFs};
    \node[box, right=2cm of arg] (isl) {ISLAFs};
    \node[box, right=2cm of isl] (imp) {imp-arg-IAFs};

    % Arrows
    \draw[dashed] (arg) --  node[above] {\footnotesize Th.\ \ref{prop:negative-simple-logic-arguments}} (isl);
    \draw (isl) --  node[above] {\footnotesize Th.\ \ref{thm:ISLAFs-lessexpressive-imp-arg-IAFs} and \ref{thm:negative-ISLAFs}} (imp);

\end{tikzpicture}
    \captionof{figure}{Expressivity results for ISLAFs. Reflexive arrows are omitted. Dashed arrows represent unknown parts of the expressivity relation. Missing arrows denote expressivity relations that have been shown not to hold.}\label{fig:ISLAFs-results}
\end{center}

\begin{remark}
    We insist that the results stated in this section apply only, as far as we have checked, to deductive argumentation with simple logic. This is because, in a simple argument $\ang{\Gamma,\alpha}$, rules are encoded in the set $\Gamma$. Hence, the presence/absence of a rule in an incomplete knowledge base is transparently linked to its presence/absence in the associated AF of all arguments containing that rule. This does not happen with other instantiations of deductive argumentation, where rules remain implicit in the background knowledge.
\end{remark}
 
%=============================
\section{Conclusions and Research Perspectives}\label{sec:conclusion}
%=============================

The core motivation of this work has been to develop a general framework for modelling uncertainty in structured argumentation formalisms and to study how this relates to existing abstract formalisms. To do so, we defined and motivated a novel notion of expressivity, informally accounting for the kinds of situations that each of these models can encode. The central focus with regard to structured formalisms has been ASPIC$^{+}$, but we have also shown how to extend our approach to other systems such as ABA and deductive argumentation. In this regard, we modelled uncertainty component-wise: only about premises and only about inference rules. \par 
Our main results show that models with uncertain rules are strictly more expressive than models with uncertain premises, and we also provided lower and upper bounds of expressivity for the abstract counterpart of these models (see Figure \ref{fig:arg-ISAFs-results}). As we show, simple arg-IAFs with no dependencies are too coarse-grained to express any of the structured formalisms we introduced. On the other side of the coin, incomplete arg-IAFs with dependencies are good enough to capture the type of uncertainty generated by structured formalisms. Indeed, weaker forms of IAFs with only implicative dependencies are sufficient to deal with uncertain rules and premises. In fact, finding precise abstract thresholds of expressivity can shed light on the precise nature of different forms of uncertainty in structured argumentation. Our work provides such thresholds for ASPIC$^{+}$ formalisms, but in some cases, these results are hardly transferable to other formalisms. The main conceptual takeaway of this work is that our results dismantle the false assumption that combinatorial uncertainty (i.e.\ uncertainty \textit{à la} IAF) at the level of components (rules, premises and preferences) is reflected at the abstract level (arguments). \par 
As for immediate future work, the current paper leaves many open questions. The most obvious is perhaps what happens if other component-wise forms of uncertainty are allowed. If we look back at the definition of ASPIC$^+$ structured argumentation framework, one could consider any of the three remaining components (languages, contrariness functions and preferences) as uncertain. It is not difficult to see that, while uncertain languages generate uncertain arguments at the abstract level, uncertain contrariness functions and preferences create uncertain defeats. In this vein, a further natural thing to ask is whether and how things change if we allow combined forms of component uncertainty (e.g.\ ASPIC$^{+}$ systems with both uncertain rules and preferences).  An interesting venue for future work is therefore to expand the approach to context with multiple forms of uncertainty.

\par

As a further point of interest, our way of modelling components' uncertainty assumes common knowledge about which category uncertain elements belong to, e.g. whether a given inference rule is strict or defeasible or whether a given premise is an axiom or a provisional assumption. Therefore, uncertainty concerns only whether an object (e.g., an inference rule) is present or absent, and not whether it is strict or defeasible. A natural research avenue consists of investigating what the consequences are of allowing uncertainty of the latter kind.

\par

Further, as noted, parts of our result do not transfer from ASPIC$^{+}$ to ABA or deductive argumentation. In particular, precise expressivity lower thresholds for these formalisms are still missing (see the dotted arrows in Figures \ref{fig:IABAFs-results} and \ref{fig:ISLAFs-results}). As a further complication here, it is known in the ABA literature that choosing among the two notions of arguments available for ABA (derivations and sets of assumptions) does not affect the result of the central reasoning task of finding sets of acceptable assumptions \cite{toni2014tutorial}. However, it can be shown that this choice does have an effect in terms of what kind of uncertain structures one obtains at the abstract level (e.g., uncertain rules generate uncertain derivation-based arguments, but they create uncertain defeats if one uses assumption-sets-arguments). A more complex question remains unanswered: what happens when acceptability and uncertainty are jointly considered within this setting? To put it in other words, does the choice about the ABA notion of argument have any effect on acceptable arguments in their necessary/possible versions?

\par

From a broader perspective, another interesting direction for future work is to investigate expressiveness not only at the level of completions, but also with respect to argumentation semantics, by comparing the sets of extensions induced by different formalisms under uncertainty, in the spirit of signature-based analyses of semantics \cite{dunne2015characteristics}. {From this perspective, perhaps some technicalities that prevented us from transferring results from ASPIC$^{+}$  to other formalisms (e.g., single versus multiple contraries) are less relevant.}
\par 
 Our inquiry was markedly theoretical. However, it should be mentioned that premise-incomplete ASPIC$^{+}$  frameworks were developed with several practical applications in mind (see \cite{odekerken2023argumentative,odekerken2025argumentative} for the relation among these systems and the work of the Dutch police). It remains an open question whether uncertain rules, our other main focus, can be applied so directly and fruitfully to real contexts. 
\par 

Finally, an interesting and relatively orthogonal direction concerns the relationship between our approach and dynamics in structured argumentation. Two comments apply here. First, premise-incomplete structured frameworks have already been studied from a dynamic perspective, insofar as the main focus of \textcite{odekerken2025argumentative} was the computational complexity of stability and relevance problems, which have a clear dynamic flavour. Second, and one step further, while the present work focuses on static representations of qualitative uncertainty via sets of completions, dynamic phenomena—such as belief revision, information update, or argument change—can be viewed as transformations on such sets. Investigating how dynamic operators on structured argumentation frameworks interact with the expressiveness notions studied here, possibly along the lines of epistemic and dynamic approaches to abstract argumentation \cite{proietti2021,doutre2018constraints}, is a natural avenue for future research. 

\section*{Acknowledgments}

A.Y.-G. gratefully acknowledges funding received from the Spanish Ministry of Science, project PID2025-172227NB-I0.
Carlo Proietti acknowledges support from the European Union under the Horizon Europe project \href{https://perycles-project.eu/}{Perycles} (Participatory Democracy that Scales).
\smallskip
\begin{center}
\includegraphics[width=0.5\textwidth]{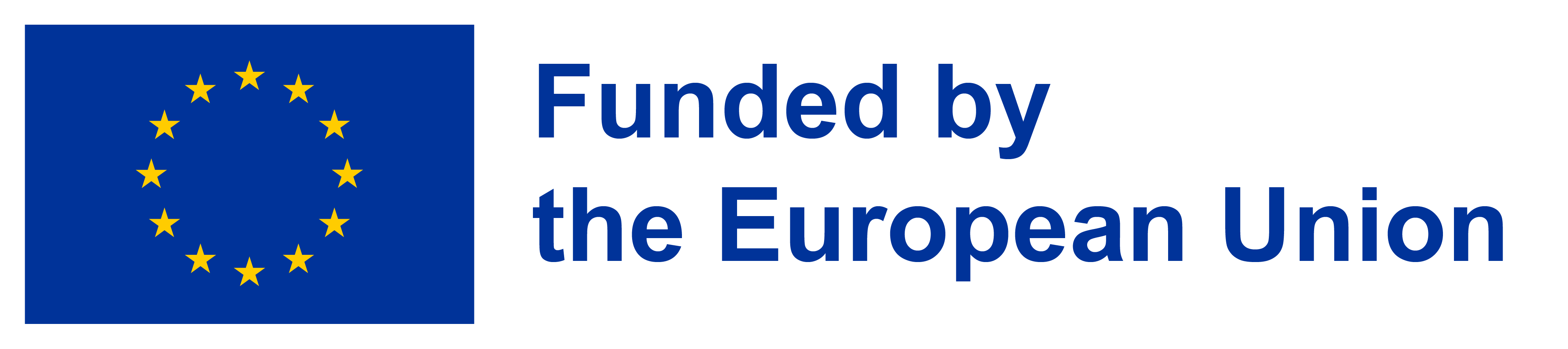}
\end{center}

The authors would also like to express their gratitude to the three anonymous reviewers of JAIR for their thorough and valuable suggestions regarding preliminary drafts of this manuscript.
\printbibliography

\end{document}